\documentclass[11pt]{article} 
\def\shownotes{1}

\usepackage[backend=biber,
            style=numeric-comp,
            natbib=true,
            maxbibnames=99,
            minalphanames=3,
            maxcitenames=2,
            sorting=ynt,
            uniquename=false,
            uniquelist=false,          
            giveninits=true,            
            sortcites=true,
            backref=true]{biblatex}
\DeclareNameAlias{sortname}{last-first}
\DefineBibliographyStrings{english}{%
  backrefpage = {page},%
  backrefpages = {pages},%
}

\usepackage[margin=1in]{geometry}
\usepackage{macro}
\usepackage[us,12hr]{datetime} %
\usepackage{graphicx, subcaption}
\usepackage{multirow, makecell}
\usepackage{algorithm}
\usepackage{algorithmic}
\usepackage{microtype}
\usepackage{sidecap}

\usepackage{amsmath,amsfonts,amssymb}
\usepackage{graphicx}
\usepackage{booktabs} %
\usepackage{enumitem}
\makeatletter
\def\thanksnosymbol#1{\protected@xdef\@thanks{\@thanks
        \protect\footnotetext{#1}}}
\makeatother
\usepackage{hyperref}
\usepackage{cleveref}
\usepackage{wrapfig}
\usepackage{breakcites}

\begin{document}

\title{Robust Fine-Tuning of Deep Neural Networks with Hessian-based Generalization Guarantees}

\author{
Haotian Ju$^{\ddag}$
\thanksnosymbol{$^{\ddag}$First two authors contributed equally. Email addresses: \texttt{\{ju.h, li.dongyu, ho.zhang\}@northeastern.edu}.}
\qquad\quad
Dongyue Li$^{\ddag}$
\qquad\quad
Hongyang R. Zhang$^{\ddag}$\\\\
{Northeastern University, Boston}
}
\date{}
\maketitle

\begin{abstract}
We consider fine-tuning a pretrained deep neural network on a target task. We study the generalization properties of fine-tuning to understand the problem of overfitting, which has often been observed (e.g., when the target dataset is small or when the training labels are noisy). Existing generalization measures for deep networks depend on notions such as distance from the initialization (i.e., the pretrained network) of the fine-tuned model and noise stability properties of deep networks. This paper identifies a Hessian-based distance measure through PAC-Bayesian analysis, which is shown to correlate well with observed generalization gaps of fine-tuned models. Theoretically, we prove Hessian distance-based generalization bounds for fine-tuned models. We also describe an extended study of fine-tuning against label noise, where overfitting remains a critical problem. We present an algorithm and a generalization error guarantee for this algorithm under a class conditional independent noise model. Empirically, we observe that the Hessian-based distance measure can match the scale of the observed generalization gap of fine-tuned models in practice. We also test our algorithm on several image classification tasks with noisy training labels, showing gains over prior methods and decreases in the Hessian distance measure of the fine-tuned model.
\end{abstract}

\section{Introduction}

Fine-tuning a large pretrained model is a common approach to performing deep learning
as foundation models become readily available. %
While this approach improves upon supervised learning in various cases, overfitting has also been observed during fine-tuning.
Understanding the cause of overfitting is challenging since dissecting the issue in practice requires a precise measurement of generalization errors for deep neural networks.
In this work, we analyze the generalization error of fine-tuned deep models using PAC-Bayes analysis and data-dependent measurements based on deep net Hessian.
With this analysis, we also study the robustness of fine-tuning against label noise.

There is a large body of work concerning generalization in deep neural networks, whereas less is understood for fine-tuning \cite{neyshabur2020being}.
A central result of the recent literature shows norm or margin bounds that improve over classical capacity bounds \cite{bartlett2017spectrally,neyshabur2017pac}.
These results can be applied to the fine-tuning setting, following a distance from the initialization perspective \cite{nagarajan2019generalization,long2020generalization}.
\citet{li2021improved} show generalization bounds depending on various norms between fine-tuned and initialization models.
These results highlight that distance from initialization crucially affects generalization for fine-tuning and informs the design of distance-based regularization to mitigate overfitting due to fine-tuning a large model on a small training set.

Our motivating observation is that in addition to distance from initialization, Hessian crucially affects generalization.
We follow the PAC-Bayesian analysis approach of \citet{neyshabur2017pac} and \citet{arora2018stronger}.
Previous work has identified noise stability properties of deep networks that require Lipschitz-continuity of the activation mappings.
We propose to quantify the stability of a deep model against noise perturbations using Hessian.
This is a practical approach as computational frameworks for deep net Hessian have been developed  \cite{yao2020pyhessian}.
We show that incorporating Hessian into the distance-based measure accurately correlates with the generalization error of fine-tuning.
See Figure \ref{table_intro} for an illustration.

\begin{figure*}
	\begin{minipage}[b]{0.38\textwidth}
		\centering
		\includegraphics[width=0.90\textwidth]{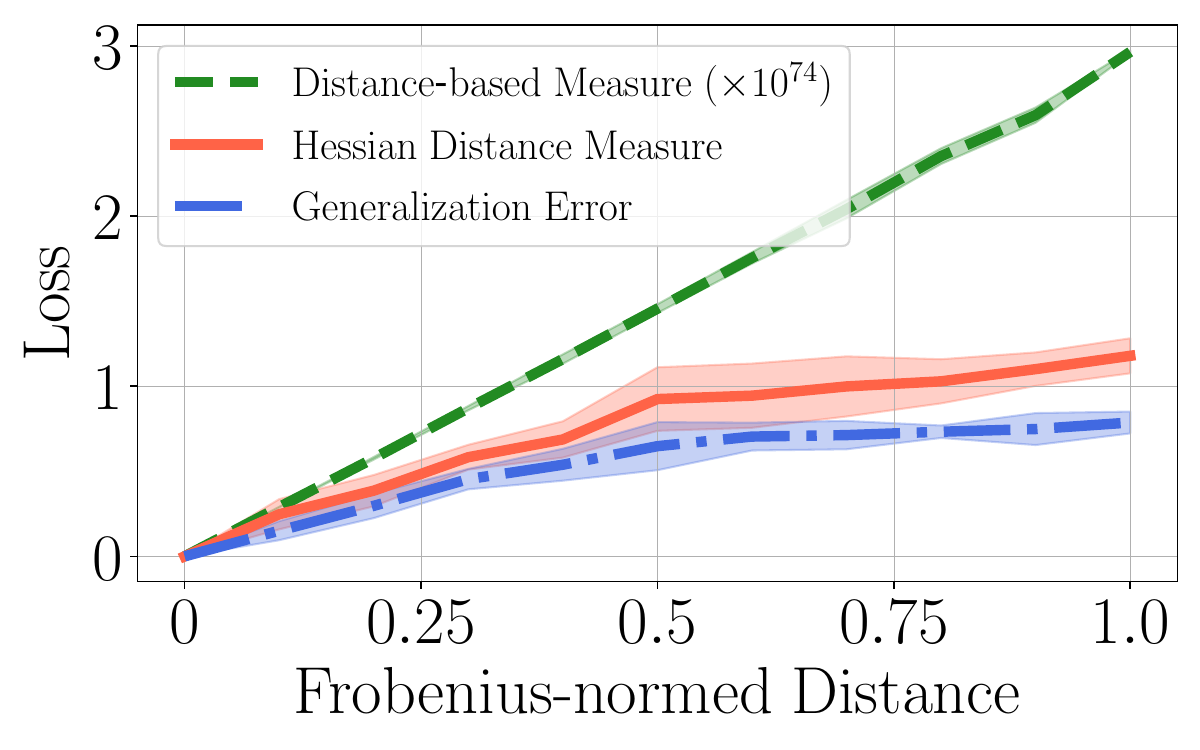}
		\caption{We identify a Hessian distance measure (cf. equation \eqref{eq_intro_hess}) that better captures empirical generalization errors of fine-tuned models.}\label{fig_intro}
	\end{minipage}\hfill
	\begin{minipage}[b]{0.58\textwidth}
		\centering
		\begin{tabular}{@{}cc@{}}
			\toprule
			Method        & Generalization error bound \\
			\hline
			Basic fine-tuning     & $\frac{\sum_{i=1}^L \sqrt{v_i^{\top} \bH_i^{+} v_i}}{\sqrt n}$ \\
			Distance-based Reg.  & $\frac{\sum_{i=1}^L \sqrt{\bigtr{\bH_i^+}\cdot \norm{v_i}^2}}{\sqrt n}$ \\
			Consistent loss w/ Reg. & $\frac{\sum_{i=1}^L \sqrt{\bignorm{(F^{-1})^{\top}}_{1,\infty}\cdot \abs{\bigtr{\bH_i}} \cdot \norm{v_i}^2}}{\sqrt n}$ \\
			\bottomrule
		\end{tabular}
	    \vspace{0.1in}
		\captionof{table}{A summary of the theoretical bounds: To use these results in practice, we can compute Hessian-vector product libraries. See Section \ref{sec_setup} for the definition of these notations.}
		\label{table_intro}
	\end{minipage}
\end{figure*}

Our theoretical contribution is to develop generalization bounds for fine-tuned multilayer feed-forward networks using Hessian.
Our result applies to many fine-tuning approaches, including vanilla fine-tuning with or without early stopping, distance-based regularization, and fine-tuning with weighted losses.
To describe the result, we introduce a few notations.
Let $f_W$ be an $L$ layer feed-forward neural network with weight matrices $W_i$ for every $i$ from $1$ to $L$.
Given a pretrained initialization $f_{W^{(s)}}$ with weight matrices $W^{(s)}_1, \dots, W^{(s)}_L$, let $v_i$ be a flatten vector of $W_i - {W}_i^{(s)}$.
For every $i = 1,\dots, L$, let $\bH_i$ be the \textit{loss} Hessian over $W_i$.
For technical reasons, denote $\bH_i^+$ as a truncation of $\bH_i$ with only non-negative eigenvalues.
We prove that on a training dataset of size $n$ sampled from a distribution $\cD$, with high probability, the generalization error of a fine-tuned $L$-layer network $f_W$ is bounded by:
\begin{align}
    \sum_{i=1}^L \sqrt{\frac{{\max_{(x, y)\sim \cD} v_i^{\top} {\bH_i^+[\ell(f_W(x), y)]} v_i}}{ n}}\label{eq_intro_hess}
\end{align}
See Theorem \ref{theorem_hessian} for the statement. Based on this result, one can derive a generalization bound for distance-based regularization of fine-tuning, assuming constraints on the norm of $v_i$ for every $i$.

Next, we study fine-tuning given noisy labels within a class conditional independent noise setting \cite{natarajan2013learning}.
Given training labels independently flipped with some probability according to a confusion matrix $F$, we extend our result to this setting by incorporating a statistically-consistent loss.
Table \ref{table_intro} summarizes the theoretical results.
Compared with the proof technique of \citet{arora2018stronger}, our proof involves a perturbation analysis of the loss Hessian, which may be of independent interest.

We extend our study to fine-tuning against label noise. We describe an algorithm that {incorporates statistically consistent losses with distance-based regularization} for improving the robustness of fine-tuning.
We evaluate this algorithm on various tasks.
First, under class-conditional label noise, our algorithm performs robustly compared to the baselines, even when the noise rate is as high as 60\%.
Second, on six image classification tasks whose labels are generated by programmatic labeling \cite{mazzetto2021adversarial}, our approach improves the prediction accuracy of fine-tuned models by {3.26\%} on average compared to existing fine-tuning approaches.
Third, for various settings, including fine-tuning vision transformers and language models, we observe similar benefits from applying our approach.
In ablation studies, we find that the strong regularization used in the method reduces the Hessian distance measure more than six times compared with several fine-tuning methods.

\subsection{Related work}\label{sec_related}

Recent work  finds that modern deep nets can fit random labels \cite{zhang2016understanding,hardt2021patterns,arora2021technical}.
Norm and margin bounds are proposed for deep nets \cite{bartlett2017spectrally,neyshabur2017pac}.
These results improve upon classical capacity notions such as the VC dimension of neural networks, which scales with the number of parameters of the network \cite{bartlett2019nearly}.
\citet{neyshabur2018towards} find that (path) norm bounds consistently correlate with empirical generalization loss for overparametrized networks.
\citet{nagarajan2018deterministic} explores noise resilience properties of deep nets to provide generalization bounds.

Recent work finds that data-dependent generalization measures better capture the empirical generalization properties of deep nets \cite{nagarajan2019generalization,wei2019improved,jiang2019fantastic}.
Geometric measures such as sharpness have been shown to correlate with empirical performance  \cite{,foret2020sharpness,tsuzuku2020normalized}.
Several studies consider numerically optimizing the prior and the posterior of the PAC-Bayes bound \cite{dziugaite2017computing,yang2022does}.
To our knowledge, provable generalization bounds that capture geometric notions such as sharpness have not been developed.
We note that the Hessian-based measure differs from the sharpness-based measure of \citet{foret2020sharpness} since the Hessian measures average perturbation, whereas sharpness measures worst-case perturbation.
Compared with existing generalization bounds, our techniques require Lipschitz-continuity of second-order gradients.
Our experiments with numerous data sets and architectures suggest that Hessian is practical for measuring generalization in deep neural networks.

Next, we discuss related approaches for studying transfer learning from a theoretical perspective.
The work of \citet{ben2010theory} and the earlier work of \citet{ben2008notion,crammer2008learning} derive generalization bounds for learning from multiple sources.
Recent works explore new theoretical frameworks and techniques to explain transfer, including  task diversity \cite{tripuraneni2020theory}, random matrix theory \cite{yang2020analysis,wu2020understanding}, and minimax estimators \cite{lei2021near}.
\citet{shachaf2021theoretical} consider a deep linear network and the impact of the similarity between the source (on which a model is pretrained) and target domains on transfer.

We use PAC-Bayesian analysis \cite{mcallester1999some} for fine-tuning, building on a distance from initialization perspective \cite{nagarajan2018deterministic}.
Early works apply PAC-Bayes bounds to analyze generalization in model averaging \cite{mcallester1999pac} and co-training \cite{dasgupta2001pac}.
Recent work has applied this approach to analyze generalization in deep networks \cite{neyshabur2017pac,arora2018stronger,nagarajan2018deterministic}.
Besides, the PAC-Bayesian analysis has been used to derive the generalization loss of graph neural networks \cite{liao2020pac} and data augmentation \cite{chatzipantazis2021learning}.
We refer the readers to a recent survey for additional references \cite{guedj2019primer}.

There is a significant body of work about learning with noisy labels, leaving a comprehensive discussion beyond the current scope.
Some related works include the design of robust losses \cite{liu2015classification}, the estimation of the confusion matrix \cite{patrini2017making,yao2020dual}, the implicit regularization of early stopping \cite{li2020gradient,yao2007early}, and label smoothing.
The work most related to ours within this literature is \citet{natarajan2013learning}, which provides Rademacher complexity-based learning bounds under class conditional label noise.
Our work builds on this work and expands their result in two aspects.
First, our bounds are derived using the PAC-Bayesian analysis, and the bounds depend on the data through the Hessian matrix.
Second, we consider deep neural networks as the hypothesis class.
Our experimental results highlight the importance of explicit regularization for improving generalization given limited and noisy labels.

\paragraph{Organization.}
The rest of this paper is organized as follows.
In Section \ref{sec_main}, we first define the problem setup formally.
Then, we present our main result following a Hessian-based PAC-Bayes analysis.
In Section \ref{sec_exp}, we present experiments to validate our proposed algorithm.
In Section \ref{sec_conclude}, we summarize our contributions and discuss several open questions for future work.
We provide the complete proofs for our results in Appendix \ref{sec_proofs}.
Lastly, we fill in missing experiment details in Appendix \ref{sec_add_exp}.

\section{Hessian-based Generalization Bounds for Fine-tuned Models}\label{sec_main}

This section presents our approach to understanding generalization in fine-tuning.
After setting up the problem formally, we present PAC-Bayesian bounds for vanilla fine-tuning and distance-based regularization.
Then, we extend the results for fine-tuning from noisy labels by incorporating a statistically consistent loss.
Lastly, we overview the proof techniques behind our main result.

\subsection{Problem setup}\label{sec_setup}

Consider predicting a target task given a training dataset of size $n$.
Denote the feature vectors and labels as $x_i$ and $y_i$, for $i = 1, \dots, n$, in which $x_i$ is $d$-dimensional and $y_i$ is a class label between $1$ to $k$.
Assume that the training examples are drawn independently from an unknown distribution $\cD$.
Let $\cX$ be the support set of the feature vectors of $\cD$.

Given a pretrained $L$-layer feed-forward network with weight matrices $\hat{W}^{(s)}_i$, for $i = 1, \dots, L$, we fine-tune the weights to solve the target task.
Let $f_W$ be an $L$-layer network initialized with $\hat{W}^{(s)}$.
The dimension of layer $i$, $W_i$, is $d_i$ by $d_{i-1}$.
Let $d_i$ be the output dimension of layer $i$.
Thus, $d_0$ is equal to the input dimension $d$, and $d_L$ is equal to the output dimension $k$.
Let $\phi_i(\cdot)$ be the activation function of layer $i$.
Given a feature vector $x\in\cX$, the output of $f_W$ is
\begin{align} %
    f_W(x) = \phi_L\left(W_L\cdot\phi_{L-1}\Big(W_{L-1}\cdots\phi_{1}\big(W_1\cdot x\big)\Big)\right).\label{eq_nn}
\end{align}%
Given a loss function $\ell: \real^k \times \set{1, \dots, k} \rightarrow \real$, let $\ell(f_W(x), y)$ be the loss of $f_W$.
The empirical loss of $f_W$, denoted by $\hat{\cL}(f_W)$, is the loss of $f_W$ averaged among the training examples.
The expected loss of $f_W$, denoted by ${\cL}(f_W)$, is the expectation of $\ell(f_W(x), y)$ over $x$ sampled from $\cD$ with label $y$.
The {generalization error} of $f_W$ is defined as its expected loss minus its empirical loss.

\paragraph{Notations.}
For any vector $v$, let $\norm{v}$ be the Euclidean norm of $v$.
For any matrix $X\in\real^{m\times n}$, let $\normFro{X}$ be the Frobenius norm and $\bignorms{X}$ be the spectral norm of $X$.
Let $\bignorm{X}_{1,\infty}$ be defined as $\max_{1\leq j\leq n}\sum_{i=1}^m\bigabs{X_{i,j}}$.
If $X$ is a squared matrix, then the trace of $X$, $\bigtr{X}$, equals the sum of $X$'s diagonal entries.
For two matrices $X$ and $Y$ with the same dimension, let $\langle X, Y\rangle = \bigtr{X^{\top}Y}$ be the matrix inner product of $X$ and $Y$.
For every $i$ from $1$ to $L$, let $\bH_{i}[\ell(f_W(x), y)]$ be a $d_id_{i-1}$ by $d_id_{i-1}$ Hessian matrix of the loss $\ell(f_W(x), y)$ over $W_i$.
Given an eigen-decomposition of the Hessian matrix, $U D U^{\top}$, let $\bH_i^+[\ell(f_W(x), y)] = U \max(D, 0) U^{\top}$ be a truncation of $\bH_i^+$ inside the positive eigen-space.
Let $v_i$ be a flatten vector of $W_i - \hat{W}_i^{(s)}$.

For two functions $f(n)$ and $g(n)$, we write $g(n) = O(f(n))$ if there exists a fixed value $C$ that does not grow with $n$ such that $g(n) \le C \cdot f(n)$ when $n$ is large enough.

\subsection{Generalization bounds using PAC-Bayesian analysis}\label{sec_pac}

To better understand what impacts fine-tuning performance, we begin by investigating it following a PAC-Bayesian approach, which has been crucial for deriving generalization bounds for deep nets \cite{neyshabur2017pac,arora2018stronger,liao2020pac}.
We consider a prior distribution $\cP$ defined as a {noisy perturbation} of the {pretrained} weight matrices $\hat{W}^{(s)}$.
We also consider a posterior distribution $\cQ$ defined as a {noisy perturbation} of the {fine-tuned} weights $\hat{W}$.
From a PAC-Bayesian perspective, the generalization performance of a deep network depends on how much noisy perturbations affect its predictions.

Previous work by \citet{arora2018stronger} has explored the noise stability properties of deep networks and showed how to derive stronger generalization bounds that depend on the properties of the network.
Let $\ell_{\cQ}(f_W(x), y)$ be the loss of $f_W(x)$, after a noisy perturbation following $\cQ$.
Denote $\ell_{\cQ}(f_W(x), y)$ minus $\ell(f_W(x), y)$ as $\cI(f_W(x), y)$:
A lower value of $\cI(f_W(x), y)$ means the network is more error resilient.
The result of \citet{arora2018stronger} shows how to quantify the error resilience of $f_W$ using Lipschitz-continuity and smoothness of the activation functions across different layers.
However, there is still a large gap between the generalization bounds achieved by \citet{arora2018stronger} and the empirical generalization errors.
Thus, a natural question is whether one can achieve provable generalization bounds that better capture empirical performance.

The key idea of our approach is to measure generalization using the Hessian of $\ell(f_W(x), y)$.
From a technical perspective, this goes beyond the previous works by leveraging Lipschitz-continuity and smoothness of the derivatives of $f_W$.
Besides, computational frameworks \cite{yao2020pyhessian} are developed for efficiently computing deep net Hessian, such as Hessian vector products \cite{sagun2016eigenvalues,papyan2018full,papyan2019measurements,ghorbani2019investigation,wu2020dissecting}.
Our central observation is that incorporating Hessian into the distance from initialization leads to provable bounds that accurately correlate with empirical performance.

To illustrate the connection between Hessian and generalization,
consider a Taylor's expansion of $\ell(f_{W+U}(x), y)$ at $W$, for some small perturbation $U$ over $W$. Suppose the mean of $U$ is and its covariance matrix is $\Sigma$. Then, $\cI(f_W(x), y)$ is equal to $\inner{\Sigma}{\bH[\ell(f_W(x),y)]}$ plus higher-order expansion terms. %
While this expansion applies to every sample, it remains to argue the uniform convergence of the Hessian.
This requires a perturbation analysis of the Hessian of all layers but can be achieved assuming Lipschitz-continuity and smoothness of the derivatives of the activation mappings.
We state our results formally below.

\begin{theorem}\label{theorem_hessian}
    Assume the activation functions $\phi_i(\cdot)$ for all $i=1, \dots, L$ and the loss function $\ell(\cdot, \cdot)$ are all twice-differentiable, and their first-order and second-order derivatives are all Lipschitz-continuous.
    Suppose $\ell(x, y)$ is bounded by a fixed value $C$ for any $x\in\cX$ with class label $y$.
    Given an $L$-layer network $f_{\hat{W}}$,
    with probability at least 0.99, for any fixed $\epsilon$ close to zero, we have
    {\begin{align}
        \cL\big(f_{\hat W}\big) \le  \big(1 + \epsilon \big) \hat{\cL}\big(f_{\hat W}\big) 
       +  {\frac{ (1 + \epsilon) \sqrt{C} \sum_{i=1}^L \sqrt{\cH_i}}{\sqrt n}} + \xi, \label{eq_main}
    \end{align}}%
    where $\cH_i$ is defined as $\max_{(x,y) \in \cD} v_i^{\top} \bH_i^+[\ell(f_{\hat{W}}(x), y)] v_i$,\footnote{In practice, to apply this result, one may set $\cD$ as the union of the training, validation, and testing datasets.} for all $i = 1,\dots, L$, and $\xi = O(n^{-3/4})$ represents an error term from the Taylor's expansion.
\end{theorem}

Theorem \ref{theorem_hessian} applies to vanilla fine-tuning, with or without early stopping.
A corollary of this result can be derived for distance-based regularization,
which restricts the distance between $W_i$ and $\hat{W}_i^{(s)}$ for every layer:
\begin{equation}
	\bignormFro{W_i - \hat{W}_i^{(s)}} \le \alpha_i,~\forall\, i = 1,\dots,L. \label{eq_constraint}
\end{equation}
Notice that $\cH_i \le \alpha_i^2 \bigbrace{\max_{(x,y)\in\cX}\tr[\bH_i^+(\ell(f_{\hat W}(x), y))]}$, the result from equation \eqref{eq_main} could then be adopted for distance-based regularization of fine-tuning.

We demonstrate that the generalization bound of Theorem \ref{theorem_hessian} accurately correlates with empirical generalization errors.
We experiment with seven methods, including fine-tuning with and without early stopping, distance-based regularization, label smoothing, mixup, etc.
Figure \ref{fig:hessian_reg} shows that the Hessian measure (i.e., the second part of equation \eqref{eq_main}) correlates with the generalization errors of these methods.
Second, we plot the value of $\cD_i$ (averaged over all layers) between fine-tuning with implicit (early stopping) regularization and explicit (distance-based) regularization.
We find that $\cD_i$ is smaller for explicit regularization than implicit regularization.

\begin{figure*}[!t]
    \begin{subfigure}[b]{0.24\textwidth}
        \centering
        \includegraphics[width=0.9\textwidth]{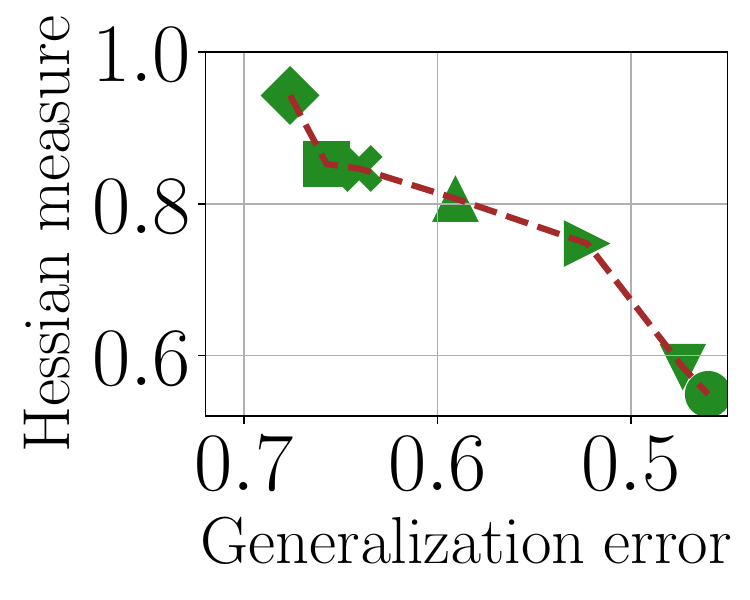}
        \caption{ResNet-50}\label{fig_reg_res1}
    \end{subfigure}\hfill
    \begin{subfigure}[b]{0.24\textwidth}
        \centering
        \includegraphics[width=0.9\textwidth]{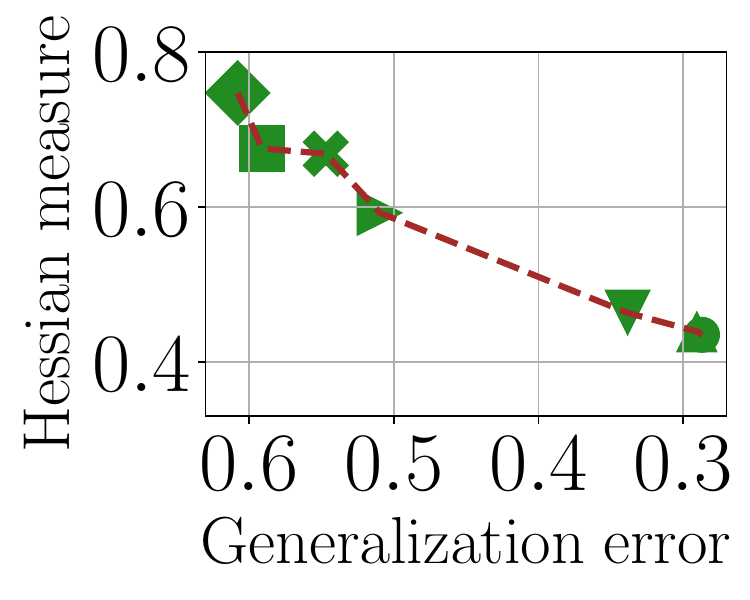}
        \caption{ResNet-50}\label{fig_reg_res2}
    \end{subfigure}\hfill
    \begin{subfigure}[b]{0.24\textwidth}
        \centering
        \includegraphics[width=0.9\textwidth]{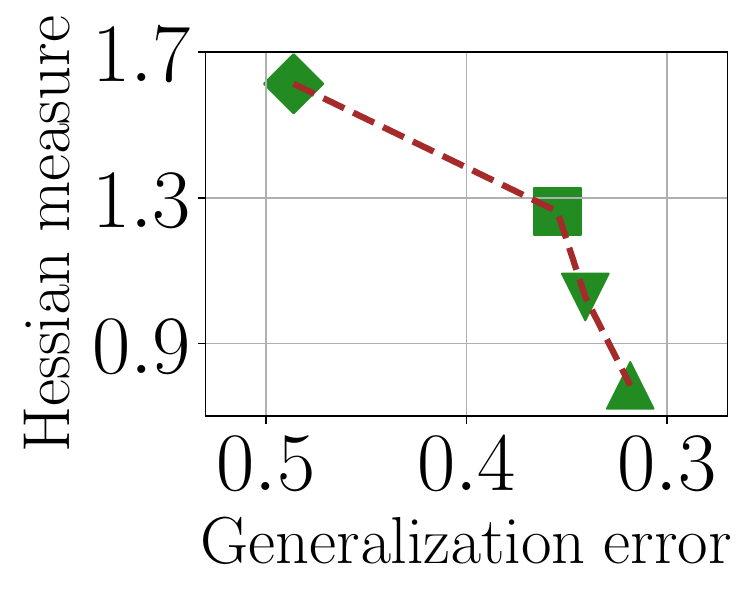}
        \caption{BERT-Base}\label{fig_reg_bert}
    \end{subfigure}\hfill
    \begin{subfigure}[b]{0.24\textwidth}
        \centering
        \includegraphics[width=0.9\textwidth]{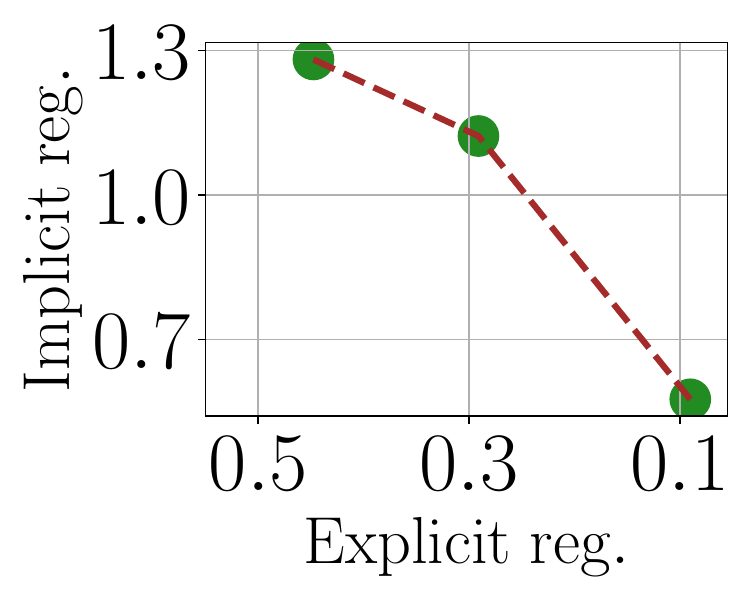}
        \caption{Average $\cD_i$}\label{fig_avg_d}
    \end{subfigure}
    \caption{
    \ref{fig_reg_res1}, \ref{fig_reg_res2}, \ref{fig_reg_bert}: The Hessian measures accurately correlate with empirical generalization errors for seven fine-tuning methods.
    \ref{fig_avg_d}: $\cD_i$ is smaller for fine-tuning with explicit distance-based regularization than implicit early-stopped regularization.
    Each dot represents the generalization error and the Hessian measures for one fine-tuned model using a particular regularization method.
    Further details are described in Appendix \ref{sec_add_exp}.}
    \label{fig:hessian_reg}
\end{figure*}

\subsection{Incorporating consistent losses with distance-based regularization}\label{sec_unbiased_loss}

Next, we consider the robustness of fine-tuning with noisy labels.
We consider a random classification noise model where in the training dataset, each label $y_i$ is independently flipped to $1,\dots,k$ with some probability.
Denote the noisy label by $\tilde y_i$.
With class conditional label noise \cite{natarajan2017cost}, $\tilde y$ is equal to $z$ with probability $F_{y, z}$, for any $z = 1,\dots, k$, where $F$ is a $k$ by $k$ confusion matrix.

Previous works \cite{natarajan2013learning,patrini2017making} have suggested minimizing statistically-consistent losses for learning with noisy labels.
Let $\bar \ell$ be a weighted loss parameterized by a $k$ by $k$ weight matrix $\Lambda$,
\begin{align}
    { \bar{\ell}(f_W(x), \tilde y) = \sum_{i=1}^k \Lambda_{\tilde y, i} \cdot \ell(f_W(x), i). }%
    \label{eq_weighted_loss}
\end{align}
It is known that for $\Lambda = F^{-1}$ (recall that $F$ is the confusion matrix), the weighted loss $\bar{\ell}$ is unbiased in the sense that the expectation of $\bar{\ell}(f_W(x), \tilde y)$ over the labeling noise is $\ell(f_W(x), y)$:
\begin{align}{
    \exarg{\tilde y}{\bar{\ell}(f_W(x), \tilde y) \mid y}
    = \ell(f_W(x), y). \label{eq_unbiased}}
\end{align}
See Lemma 1 in \citet{natarajan2013learning} for the binary setting and Theorem 1 in \citet{patrini2017making} for the multi-class setting.
As a result, a natural approach to designing a theoretically-principled fine-tuning algorithm is to minimize weighted losses with $\Lambda = F^{-1}$.
However, this approach does not take overfitting into consideration.

Our proposed approach involves incorporating consistent losses with distance-based regularization.
Let $\bar{\cL}(f_W)$ be the average of the weighted loss $\bar{\ell}(f_W(x_i), \tilde y_i)$ over $i = 1,\dots, n$.
We extend Theorem \ref{theorem_hessian} to this setting as follows.

\begin{theorem}\label{theorem_noisy}
    Assume the activation functions $\set{\phi_i(\cdot)}_{i=1}^L$ and the loss function $\ell(\cdot,\cdot)$ satisfy the assumptions stated in Theorem \ref{theorem_hessian}.
    Suppose the noisy labels are independent of the feature vector conditional on the class label.
    Suppose the loss function $\ell(x, z)$ is bounded by $C$ for any $x\in\cX$ and any $z\in\set{1,\dots,k}$.
    With probability $0.99$, for any fixed $\epsilon$ close to zero, we have
    {\begin{align*}
        \cL\big(f_{\hat W}\big)
        \le& (1 + \epsilon) \bar{\cL}\big(f_{\hat W}\big) \\
        &+ \frac{(1 + \epsilon) \sqrt{{C \bignorm{(F^{-1})^{\top}}_{1,\infty}}}  \sum_{i=1}^L \sqrt{\alpha_i^2 \max_{x\in\cX, y\in\set{1,\dots,k}} \abs{\tr[\bH_i(\ell(f_{\hat W}(x), y))]}}}{\sqrt n} + \xi,
    \end{align*}}%
    where $\alpha_i$ is any value greater than $\normFro{\hat W_i - \hat W_i^{(s)}}$, for any $i = 1,\dots, L$, and $\xi = O\big(n^{-3/4}\big)$ represents an error term from the Taylor's expansion.
\end{theorem}

Theorem \ref{theorem_noisy} shows that combining consistent losses leads to a method that is provably robust under class conditional label noise.
Notice that this result uses the trace of the Hessian $\bH_i$ instead of the truncated Hessian $\bH_i^+$. The reweighing matrix $\Lambda$ might include negative coefficients, and $\bH_i^+$ deals with this problem.
Based on the theory, we design an algorithm to instantiate this idea.
In Section \ref{sec_exp}, we evaluate the robustness of Algorithm \ref{alg:reg_weighted_loss} for both image and text classification tasks.

\begin{algorithm}[!t]
	\caption{Statistically-consistent loss re-weighting with layer-wise projection}\label{alg:reg_weighted_loss}
	\begin{small}
		\textbf{Input}: Training dataset $\{(x_i, \tilde{y}_i)\}_{i=1}^{n}$ with input feature $x_i$ and noisy label $\tilde y_i$, for $i = 1,\dots, n$. \\
		\textbf{Require}: Initialization $f_{\hat W^{(s)}}$, layer-wise distance $\alpha_i$ for $i = 1,\dots,L$, number of epochs $T$, learning rate $\eta$, and confusion matrix $F$. \\
		\textbf{Output}: A trained model $f_{W^{(T)}}$.
		\begin{algorithmic}[1]
			\STATE Let $\Lambda = F^{-1}$.
			\STATE At $t = 0$, initialize model parameters with weight matrices $W^{(0)} = \hat{W}^{(s)}$.
			\WHILE{$0 < t < T$}
			\STATE Let $\bar{\cL}(f_{W^{(t-1)}})  = \frac 1 n \sum_{i=1}^n \sum_{c=1}^k \Lambda_{\tilde{y}_i, c} \cdot \ell(f_{W^{(t-1)}}(x_i), c)$.
			\STATE Update $W^{(t)} \leftarrow W^{(t-1)} - \eta \cdot \nabla \bar{\cL}(f_{W^{(t-1)}})$ using stochastic gradient descent.
			\STATE Project $W_i^{(t)}$ inside the constrained ball: $W_i^{(t)} \leftarrow \min\Big(1, \frac{\alpha_i}{\normFro{W_i^{(t)} - W_i^{(0)}}}\Big)\big(W_i^{(t)} - W_i^{(0)}\big) + W_i^{(0)}$, $\forall\, i=1,\dots,L$.
			\ENDWHILE
		\end{algorithmic}
	\end{small}
\end{algorithm}

\subsection{Proof overview}\label{sec_proof_overview}

We give an overview of the proof of Theorems \ref{theorem_hessian} and \ref{theorem_noisy}.
First, we show that the noise stability of an $L$-layer neural network $f_W$ admits a layer-wise Hessian approximation.
Let $U \sim \cQ$ be a random variable drawn from a posterior distribution  $\cQ$.
We are interested in the perturbed loss, $\ell_{\cQ}(f_U(x),y)$, which is the expectation of $\ell(f_{U}(x),y)$ over $U$.
\begin{lemma}\label{lemma_perturbation}
    In the setting of Theorem \ref{theorem_hessian}, for any $i=1, 2, \cdots, L$, let $U_i \in \real^{d_i d_{i-1}}$ be a random vector sampled from a Gaussian distribution with mean zero and variance $\Sigma_i$.
    Let the posterior distribution $\cQ$ be centered at $W_i$ and perturbed with an appropriately reshaped $U_i$ at every layer.
    Then, there exists a fixed value $C_1 > 0$ that does not grow with $n$, such that the following holds for any $x\in\cX$ and $y\in\set{1,\dots,k}$:
    \begin{align}
        \ell_\cQ(f_W(x),y) - \ell(f_W(x),y)
        \leq \sum_{i=1}^L \left( \big\langle \Sigma_i,\mathbf{H}_i[\ell(f_W(x),y)]\big\rangle + C_1\bignormFro{\Sigma_i}^{3/2} \right). \label{eq_taylor}
    \end{align}
\end{lemma}
See Appendix \ref{proof_perturbation} for the proof of Lemma \ref{lemma_perturbation}.
Interestingly, we find that the Hessian estimate is remarkably accurate in practice.
We use ResNet-18 and ResNet-50, fine-tuned on the CIFAR-100 dataset.
We estimate the left-hand side of equation \eqref{eq_taylor} by randomly sampling $500$ isotropic Gaussian perturbations and averaging the perturbed losses.
Then, we compute the Hessian term on the right-hand side of equation \eqref{eq_taylor} by measuring the traces of the loss's Hessian at each layer of the network and average over the same data samples.
Table \ref{table:measure_noise_stability} shows the result.
The relative error of the Hessian estimate is within 3\% on ResNet.

\begin{table*}[t]
\centering
\caption{Showing the relative residual sum of squares (RSS) of the Hessian approximation \eqref{eq_taylor} under isotropic noise perturbation of variance $\sigma^2$. The noise stability results are estimated by the average of 500 noise injections. All the experiments are done with the CIFAR-100 dataset.}\label{table:measure_noise_stability}
\begin{small}
\begin{tabular}{@{}lcccc@{}}
\toprule
\multirow{2}{*}{$\sigma$} & \multicolumn{2}{c}{ResNet-18} &  \multicolumn{2}{c}{ResNet-50} \\ \cmidrule(l){2-3} \cmidrule(l){4-5} 
& Noise Stability  & Hessian Approx.                   & Noise Stability  & Hessian Approx.    \\ \midrule
0.01 &   $0.86	\pm 0.19$ &   1.07 &  $0.39 \pm 0.10$     &  1.24   \\
0.011 &  $1.13	\pm 0.24$ &  1.29 & $1.12	\pm 0.20$  &  1.50  \\
0.012 & $1.45	\pm 0.29$  &  1.54 & $1.56	\pm 0.26$  &  1.79   \\
0.013 &  $1.82	\pm 0.35$ &   1.80 &  $2.10	\pm 0.33$  & 2.10   \\
0.014 & $2.23	\pm 0.40$  &   2.09 &  $2.71	\pm 0.38$  &  2.43   \\
0.015 & $2.65	\pm 0.43$  &  2.34 & $3.30	\pm 0.40$  & 2.79    \\
0.016 & $3.07	\pm 0.45$  &  2.73 & $3.80	\pm 0.40$  & 3.18    \\
0.017 & $3.47	\pm 0.47$  &  3.08 & $4.17	\pm 0.40$  &  3.59   \\
0.018 &  $3.84\pm	0.49$ &  3.46 & $4.60	\pm 0.44$  & 4.02   \\
0.019 & $4.15\pm	0.51$  &  3.85 & $4.77	\pm 0.44$  &  4.48  \\
0.020 &  $4.43\pm	0.55$ &  4.27 & $5.03 \pm 0.82$    & 4.97   \\ \midrule
Relative RSS & \multicolumn{2}{c}{\textbf{0.75\%}} &  \multicolumn{2}{c}{\textbf{2.98\%}} \\\bottomrule
\end{tabular}
\end{small}
\end{table*}

Based on the Hessian approximation, consider applying a PAC-Bayes bound over the $L$-layer network $f_W$ (see Theorem \ref{lemma_pac} in Appendix \ref{sec_proofs} for a statement).
The KL divergence between the prior distribution and the posterior distribution is equal to
\begin{equation}
    \sum_{i=1}^L \biginner{\Sigma_i^{-1}}{v_i v_i^{\top}}. \label{eq_kl_sum}
\end{equation}
Minimizing the sum of the Hessian estimate and the above KL divergence \eqref{eq_kl_sum} in the PAC-Bayes bound will lead to a different covariance matrix for every layer, which depends on $v_i$ and $\ex{\bH_i^+}$:
\begin{align*}
    & \sum_{i=1}^L \Big(\inner{\Sigma_i}{\exarg{x,y}{\bH_i[\ell(f_W(x), y)]}} + \frac 1 n \inner{\Sigma_i^{-1}}{v_i v_i^{\top}}\Big) \\
    \le & \sum_{i=1}^L \Big( \inner{\Sigma_i}{\exarg{x,y}{\bH_i^+[\ell(f_W(x), y)]}} + \frac 1 n \inner{\Sigma_i^{-1}}{v_i v_i^{\top}} \Big),
\end{align*}
where the $1/n$ factor is inherited from the PAC-Bayes bound (cf. Theorem \ref{lemma_pac}).
Thus, minimizing the above over the covariance matrices leads to the Hessian distance-based generalization bound of Theorem \ref{theorem_hessian}. %
We note that this minimization yields a deterministic solution for the covariance matrix; this solution is fixed so that it does not change with the randomness of the training data.

The above sketch highlights the crux of the result.
The rigorous proof, on the other hand, is significantly more involved.
One technical challenge is arguing the uniform convergence to the population Hessian operator $\bH_i$ (over $\cD$) from the empirical loss Hessian.
This requires a perturbation analysis of the Hessian and assumes that the second-order derivatives of the activation functions (and the loss) are Lipschitz-continuous.

\begin{lemma}\label{lemma_union_bound}
    In the setting of Theorem \ref{theorem_hessian}, there exist  some fixed values $C_2,C_3$ that do not grow with $n$ and $1/\delta$, such that with probability at least $1- \delta$ over the randomness of the training set, we have
    \begin{align}
        \bignormFro{\frac 1 n \sum_{j=1}^n\bH_i[\ell(f_W(x_j), y_j)] - \exarg{(x,y)\sim\cD}{\bH_i[\ell(f_W(x), y)]}} \le \frac{C_2\sqrt{\log (C_3 n/\delta)}}{\sqrt n},
    \end{align}
    for any $i =1,\dots, L$, where $C_2 = 4\cH_i\sqrt{\sum_{i=1}^L d_i d_{i-1}}$, $C_3 = 5L\Big(\sqrt{\sum_{i=1}^L \normFro{W_i}^2}\Big)G/\cH_i$ and $G$ is a fixed value that depends on $\prod_{i=1}^L \norm{W_i}_2$, the Lipschitz-continuity of the activation functions, and the Lipschitz-continuity of their first-order and second-order derivatives (cf. equation \eqref{eq_G_constant} in Proposition \ref{prop_hessian_lip} for the precise value of $G$).
\end{lemma}

See Appendix \ref{proof_union} for the proof of Lemma \ref{lemma_union_bound}, based on an $\epsilon$-net argument. The rest of the proof of Theorem \ref{theorem_hessian} can be found in Appendix \ref{sec_hessian_proof}.

The key idea for dealing with noisy labels is a consistency condition for the weighted loss $\bar\ell(f_W(x), \tilde y)$ with noisy label $\tilde y$:
{\[ \cL(f_W) = \mathop{\mathbb E}_{(x, y)\sim \cD}\exarg{\tilde y | y}{\bar{\ell}(f_W(x), \tilde y)}. \]}%
Thus, consider a noisy data distribution $\tilde D$, with the label $y$ of every $x$ flipped independently to $\tilde y$.
Our key observation is that $\bar{\ell}(f_W(x), \tilde y)$ enjoys similar noise stability properties as $\ell(f_W(x), y)$.
See Figure \ref{fig:heatmap} for an illustration.
A complete proof of Theorem \ref{theorem_noisy} can be found in Appendix \ref{sec_proof_noisy}.

\begin{figure*}[!t]
	\centering
	\begin{subfigure}[b]{0.24\textwidth}
		\centering
		\includegraphics[width=0.95\textwidth]{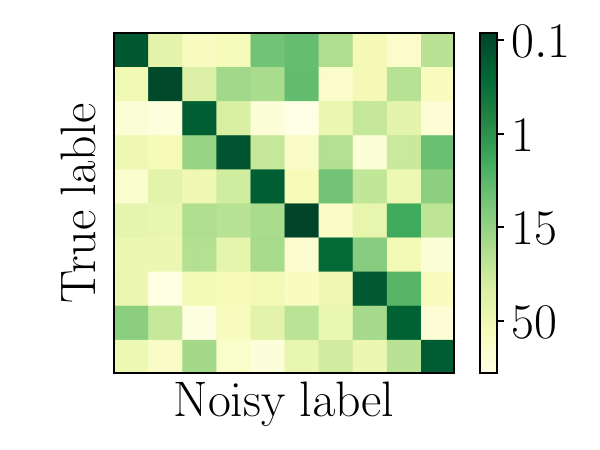}
		\captionsetup{justification = centering}
		\caption{ResNet-18 on \\ CIFAR-10}
	\end{subfigure}%
	\begin{subfigure}[b]{0.24\textwidth}
		\centering
		\includegraphics[width=0.95\textwidth]{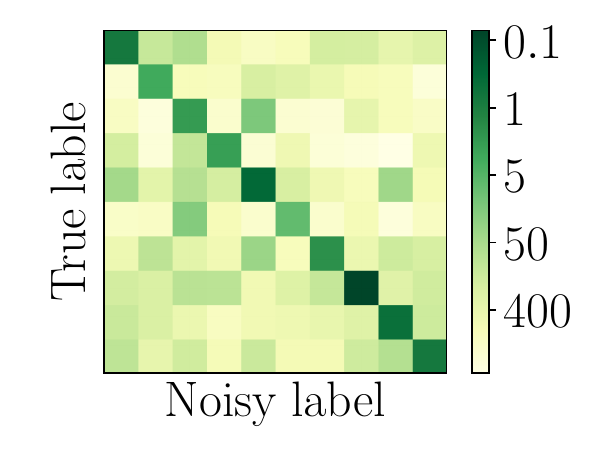}
		\captionsetup{justification = centering}
		\caption{ResNet-50 on \\ CIFAR-10}
	\end{subfigure}%
	\begin{subfigure}[b]{0.24\textwidth}
		\centering
		\includegraphics[width=0.95\textwidth]{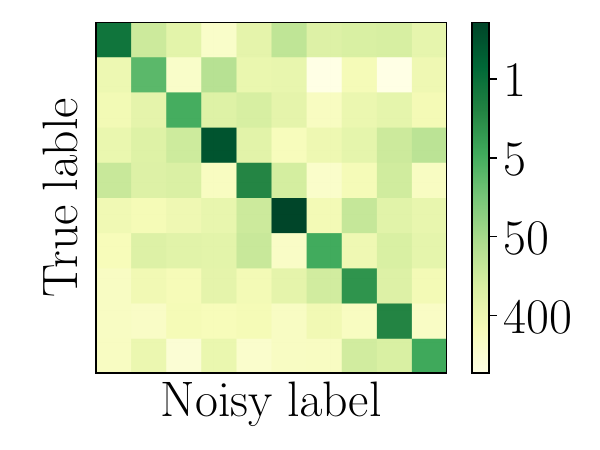}
		\captionsetup{justification = centering}
		\caption{ResNet-18 on \\ CIFAR-100}
	\end{subfigure}%
	\begin{subfigure}[b]{0.24\textwidth}
		\centering
		\includegraphics[width=0.95\textwidth]{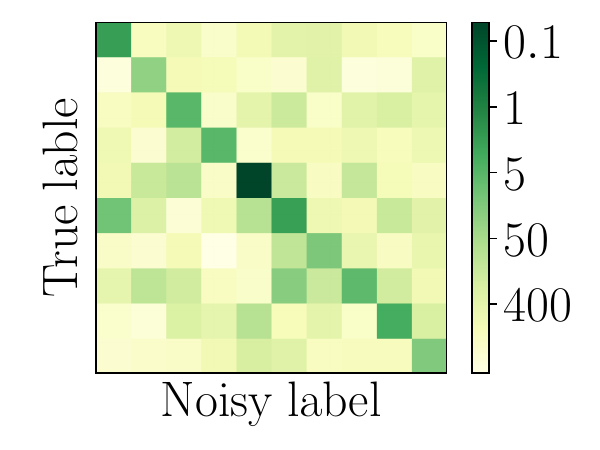}
		\captionsetup{justification = centering}
		\caption{ResNet-50 on \\ CIFAR-100}
	\end{subfigure}
	\caption{Illustrating the intuition behind Theorem \ref{theorem_noisy}: With consistent losses, the noisy stability of the loss Hessian with noisy labels equals the noise stability of the loss Hessian with class labels in expectation. Thus, the unbiased weighted losses enjoy similar noise stability properties as training with class labels. Further, the unbiased loss favors stable models as the trace of the Hessian is smallest when $\tilde y = y$.
	}\label{fig:heatmap}
\end{figure*}

\section{Experiments}\label{sec_exp}

This section conducts experiments to evaluate our proposed algorithm.
We evaluate the robustness of Algorithm \ref{alg:reg_weighted_loss} under various noisy environments and architectures.
First, we validate that when the training labels are corrupted with class conditional label noise, our approach is robust even under up to 60\% of label noise.
Second, on six weakly-supervised image classification tasks, our algorithm improves the prediction accuracy over prior approaches by 3.26\% on average.
Lastly, we present a detailed analysis that validates the theory, including comparing our generalization bound with previous results and comparing the Hessian distance measure for various fine-tuning methods.
The code repository for reproducing the experiments is at \url{https://github.com/VirtuosoResearch/Robust-fine-tuning}. 

\subsection{Experimental setup}\label{sec_exp_setup}

\noindent\textbf{Datasets.} 
We evaluate the robustness of our approach in image and text classification tasks. 
For image classification, we use six domains of object classification tasks from the DomainNet dataset.
We consider two kinds of label noise. First, we generate synthetic random noise by randomly flipping the labels of training samples uniformly with a given noise rate.
Second, we create labels using weak supervision approaches \cite{mazzetto2021adversarial}. %
The statistics of the six datasets are described in Table \ref{tab:dataset_statistics} from Section \ref{sec_add_exp_noisy}.
We refer the reader to the work of \citet{mazzetto2021adversarial} for more details.

For text classification, we use the MRPC dataset from the GLUE benchmark, which is to predict whether two sentences are semantically equivalent.

\paragraph{Models.} For fine-tuning from noisy labels, we use pretrained ResNet-18 and ResNet-101 models on image classification tasks and extend our results to Vision Transformer (ViT) model. We use the RoBERTa-Base model on text classification tasks.

\paragraph{Baselines.} %
We consider the following three kinds of baselines: 
(i) Regularization:  Fine-tuning with early stop (Early stopping), label smoothing, Mixup \cite{wu2020generalization}, and Early-Learning Regularization (ELR) \cite{liu2020early};
(iii) Self-training:  FixMatch \cite{sohn2020fixmatch}; %
(ii) Robust loss function: Active Passive Loss (APL) \cite{ma2020normalized}, dual transition estimator (DualT) \cite{yao2020dual}, Supervised Contrastive learning (SupCon), and Sharpness-Aware Minimization (SAM) \cite{foret2020sharpness}.

To implement our algorithm, we use the confusion matrix $F$ estimated by the method from the work of \citet{yao2020dual}. 
We use layer-wise distances \cite{li2021improved} for applying the regularization constraints in Equation \eqref{eq_constraint}. 
We describe the implementation details and hyper-parameters in Section \ref{sec_add_exp_noisy}.

\subsection{Experimental results}\label{sec_exp_noisy}

\begin{table*}[!t]
\centering
\caption{Test accuracy of fine-tuning with noisy labels. We consider two settings: one where the training labels are independently flipped to incorrect labels and another where the training labels are created using programmatic labeling approaches. The reported results are averaged over ten random seeds.}\label{tab:syn_and_ws_noise}
\begin{scriptsize}
\begin{tabular}{@{}lccccccc@{}}
\toprule
\multirow{2}{*}{Independent label noise}  & \multicolumn{2}{c}{Sketch, ResNet-18} & \multicolumn{2}{c}{Sketch, ResNet-101}
  & \multicolumn{2}{c}{MRPC, RoBERTa-Base}\\ \cmidrule(l){2-7} 
  & 40\% & 60\% & 40\%  & 60\%  & 20\% & 40\%   \\ \midrule
Early Stopping                      & 72.41$\pm$3.53 & 53.84$\pm$3.09 & 77.14$\pm$3.09 & 61.39$\pm$1.28 & 81.39$\pm$0.73 & 66.05$\pm$0.63  \\
Label Smooth  & 74.69$\pm$1.97 & 55.35$\pm$1.60 & 81.47$\pm$1.36 & 64.90$\pm$2.93 & 80.00$\pm$0.62 & 65.87$\pm$0.95  \\
Mixup          & 70.65$\pm$1.85 & 58.49$\pm$3.25 & 76.04$\pm$2.29 & 60.12$\pm$2.37 & 80.62$\pm$0.14 & 68.37$\pm$1.33  \\
FixMatch \cite{sohn2020fixmatch}    & 73.35$\pm$3.15 & 61.51$\pm$2.17 & 76.19$\pm$1.39 & 62.19$\pm$1.57 & 81.10$\pm$1.76 & 68.48$\pm$1.43  \\
ELR \cite{liu2020early}             & 74.29$\pm$2.52 & 63.14$\pm$2.05 & 80.00$\pm$1.45 & 64.35$\pm$3.98 & 82.78$\pm$0.86 & 67.86$\pm$1.44  \\
APL \cite{ma2020normalized}         & 75.63$\pm$1.81 & 64.69$\pm$2.72 & 78.69$\pm$2.45 & 64.82$\pm$2.41 & 80.49$\pm$0.24 & 66.49$\pm$0.93  \\
DualT \cite{yao2020dual}            & 72.49$\pm$3.17 & 59.59$\pm$3.44 & 77.96$\pm$0.33 & 62.31$\pm$3.98 & 82.49$\pm$0.53 & 66.49$\pm$0.93  \\
SupCon    & 75.14$\pm$1.73 & 61.06$\pm$3.20 & 78.86$\pm$1.80 & 63.92$\pm$2.15 & 82.30$\pm$1.80 & 68.32$\pm$1.16  \\
\midrule
\textbf{Ours (Alg. \ref{alg:reg_weighted_loss})}                                & \textbf{81.96$\pm$0.98} & \textbf{70.00$\pm$1.71} & \textbf{85.44$\pm$1.26} & \textbf{71.84$\pm$2.72} & \textbf{83.55$\pm$0.63} & \textbf{72.64$\pm$1.77}  \\ 
\midrule
\midrule
\multirow{3}{*}{Correlated label noise}  & \multicolumn{6}{c}{DomainNet, ResNet-18}  \\ \cmidrule(l){2-7} 
  & Clipart  & Infograph & Painting  &  Quickdraw & Real  & Sketch  \\
  & 41.47\% & 63.29\% & 44.50\% & 60.54\% & 34.64\% & 47.68\%  \\
  \midrule
Early Stopping & 73.88$\pm$2.04 & 38.82$\pm$2.59 & 69.69$\pm$1.35 & 44.16$\pm$1.92 & 78.52$\pm$1.03 & 61.84$\pm$3.67 \\
Label Smooth &  74.56$\pm$2.30 & 38.40$\pm$2.67 & 70.76$\pm$1.74 & 46.50$\pm$3.03 & 81.39$\pm$0.93 & 62.29$\pm$2.48 \\
Mixup  & 72.88$\pm$0.94 & 39.27$\pm$3.10 & 69.28$\pm$3.18 & 47.66$\pm$3.20 & 80.17$\pm$2.05 & 62.08$\pm$3.06 \\
FixMatch \cite{sohn2020fixmatch} & 77.04$\pm$2.52 & 41.95$\pm$1.52 & 73.31$\pm$2.10 & 48.74$\pm$2.08 & 86.33$\pm$2.54 & 64.61$\pm$3.28  \\
ELR \cite{liu2020early} & 76.08$\pm$2.03 & 40.14$\pm$2.74 & 72.06$\pm$1.73 & 47.40$\pm$3.09 & 83.64$\pm$2.09 & 65.76$\pm$3.19  \\
APL \cite{ma2020normalized} & 77.40$\pm$2.33 & 41.22$\pm$2.58 & 73.61$\pm$3.12 & 49.88$\pm$3.24 & 85.79$\pm$1.59 & 64.69$\pm$2.30 \\
DualT \cite{yao2020dual} & 75.24$\pm$2.02 & 38.75$\pm$2.12 & 70.27$\pm$2.24 & 46.62$\pm$3.16 & 83.33$\pm$3.01 & 65.47$\pm$1.91 \\
SupCon  & 76.56$\pm$3.53 & 40.38$\pm$1.94 & 72.51$\pm$2.45 & 49.20$\pm$2.63 & 81.87$\pm$0.84 & 65.67$\pm$2.90 \\
\midrule
\textbf{Ours (Alg. \ref{alg:reg_weighted_loss})}    & \textbf{83.28$\pm$1.64} & \textbf{43.38$\pm$2.45} & \textbf{76.32$\pm$1.08} & \textbf{50.32$\pm$2.74} & \textbf{92.36$\pm$0.78} & \textbf{66.86$\pm$3.29} \\ \bottomrule
\end{tabular}
\end{scriptsize}
\end{table*}

\noindent\textbf{Improved robustness against independent label noise.} 
Our result in Section \ref{sec_unbiased_loss} illustrates why our algorithm might be expected to achieve provable robustness under class conditional label noise. 
To validate its effectiveness, we apply our algorithm to image and text classification tasks with different levels of synthetic random labels. Table \ref{tab:syn_and_ws_noise} reports the test accuracy. 
\begin{itemize}
\item First, we fine-tune ResNet-18 on the Sketch domain from the DomainNet dataset with 40\% and 60\% synthetic noise. Our algorithm improves upon the baseline methods by  \textbf{6.43\%} on average. 
The results corroborate our theory that provides generalization guarantees for our algorithm.  
\item Second, to evaluate our algorithm with deeper models, we apply ResNet-101 on the same label noise settings. We observe a higher average performance with ResNet-101, and our algorithm achieves a similar performance boost of \textbf{6.03\%} on average over the previous methods.
\item Third, we fine-tune the RoBERTa-Base model on the MRPC dataset with 20\% and 40\% synthetic noise. Our algorithm outperforms previous methods by \textbf{2.91\%} on average. These results show that our algorithm is robust against label noise across various tasks and architectures.
\end{itemize}

\paragraph{Improved robustness against correlated label noise.} 
Next, we show that our algorithm can achieve competitive performance under real-world label noise.
We evaluate our algorithm on weakly-supervised image classification tasks.
We fine-tune ResNet-18 on the six domains from the DomainNet dataset and report the test accuracy in Table \ref{tab:syn_and_ws_noise}.
We find that Algorithm \ref{alg:reg_weighted_loss} outperforms previous methods by \textbf{3.26\%}, averaged over the six tasks.

We expect our results under correlated label noise to hold with various model architectures. 
As an extension of the above experiment, we fine-tune vision transformer models on the same set of tasks.
Table \ref{tab:vit_results} reports the results on the Clipart and Sketch domains.
We notice that using ViT as the base model boosts the performance across all approaches significantly (e.g., over 4\% for early stopping).
Our approach outperforms baseline methods by \textbf{2.42\%} on average.

\begin{table*}[!t]
\centering
\captionof{table}{Test accuracy of fine-tuning vision transformers with correlated noisy labels on Clipart and Sketch domains from the DomainNet dataset. Results are averaged over ten random seeds.}\label{tab:vit_results}
\begin{scriptsize}
\begin{tabular}{@{}lcccccc@{}}
\toprule
\multirow{2}{*}{Correlated label noise}  & \multicolumn{2}{c}{DomainNet, ViT-Base} \\ \cmidrule(l){2-3} 
 & Clipart   & Sketch \\ \midrule
Early Stopping & 80.48$\pm$2.65 & 66.29$\pm$3.82 \\
Label Smooth  &  80.00$\pm$3.23 & 67.02$\pm$1.38  \\
Mixup    & 81.44$\pm$2.33 & 69.06$\pm$1.05  \\
FixMatch \cite{sohn2020fixmatch} & 81.92$\pm$1.35 & 68.08$\pm$2.52  \\
ELR  \cite{liu2020early}    & 82.48$\pm$2.32 & 68.33$\pm$3.27  \\
APL  \cite{ma2020normalized}    & 79.12$\pm$3.75 & 68.33$\pm$2.39  \\
DualT \cite{yao2020dual}   & 82.24$\pm$3.55 & 67.92$\pm$2.97 \\
SupCon      & 82.64$\pm$2.80 & 65.22$\pm$2.18  \\
\midrule
\textbf{Ours (Alg. \ref{alg:reg_weighted_loss})} & \textbf{84.40$\pm$1.94} & \textbf{71.02$\pm$1.81} \\ \bottomrule
\end{tabular}
\end{scriptsize}
\end{table*}

\paragraph{Ablation study.} We study the influence of each component of our algorithm: the weighted loss scheme and the regularization constraints.
We ablate the effect of our weighted loss and regularization constraints, showing that both are crucial to the final performance.
We run the same experiments with only one component from the algorithm. As shown in Table \ref{tab:ablation_study}, removing either component negatively affects performance.
This suggests that the weighted loss and the regularization constraints are important in the final performance.

\begin{table*}[h]
\centering
\caption{Removing either component in Algorithm \ref{alg:reg_weighted_loss} degrades the test accuracy. Results are averaged over ten random seeds.}\label{tab:ablation_study}
\begin{scriptsize}
\begin{tabular}{@{}lcccccc@{}}
\toprule
\multirow{2}{*}{Correlated label noise}  & \multicolumn{6}{c}{DomainNet, ResNet-18} \\ \cmidrule(l){2-7} 
  & Clipart  & Infograph  & Painting  &  Quickdraw & Real  & Sketch \\ \midrule
Using only loss reweighting & 77.96$\pm$1.57 & 41.78$\pm$2.86 & 72.13$\pm$1.66 & 49.10$\pm$2.42 & 85.34$\pm$1.51 & 64.78$\pm$1.60 \\
Using only regularization &  82.84$\pm$1.11 & 42.34$\pm$1.92 & 75.81$\pm$2.07 & 49.76$\pm$2.86 & 91.67$\pm$0.85 & 65.96$\pm$2.73 \\
\textbf{Ours (Alg. \ref{alg:reg_weighted_loss})} & \textbf{83.28$\pm$1.64} & \textbf{43.38$\pm$2.45} & \textbf{76.32$\pm$1.08} & \textbf{50.32$\pm$2.74} & \textbf{92.36$\pm$0.78} & \textbf{66.86$\pm$3.29} \\ \bottomrule
\end{tabular}
\end{scriptsize}
\end{table*}

\subsection{Further numerical results}\label{sec:exp_analysis}

\paragraph{Comparison of Hessian-based distance measures.} 
The design of our algorithm is motivated by our Hessian-based generalization guarantees. 
We hypothesize that our algorithm can effectively reduce the Hessian distance measure $\sum_{i=1}^L\sqrt{\cH_i}$  (cf. Equation \ref{eq_main}) of fine-tuned models. 
We compare the Hessian quantity of models fine-tuned by different algorithms. We select five baseline methods to compare with our algorithm and expect similar results for comparison with other baselines.
Figure \ref{fig:reduced_hessians} shows the results of ResNet-18 models fine-tuned on the Clipart dataset with weak supervision noise. Our algorithm reduces the quantity by six times more than previous fine-tuning algorithms.

\paragraph{Comparison of generalization bounds.} 
Our result suggests that Hessian-based generalization bounds correlate with the empirical performance of fine-tuning.
Next, we compare our result with previous generalization bounds, including norm bounds \cite{long2020generalization,gouk2020distance,li2021improved}, and margin bounds \cite{neyshabur2017pac,pitas2017pac,arora2018stronger}.
See Appendix \ref{sec_add_setup} for references on the measurement quantities.
For our result, we numerically calculate ${\sum_{i=1}^L \sqrt{C\cdot \cH_i}/{\sqrt n}} + \xi$ from equation \ref{eq_main}.
We compute $\cH_i$ using Hessian vector product functions in PyHessian \cite{yao2020pyhessian}. We take the maximum over the training and test dataset.
Following prior works, we evaluate the results using the CIFAR-10 and CIFAR-100 data sets.
We use ResNet-50 as the base model for fine-tuning.
Additionally, we evaluate our results on text classification tasks, including MRPC and SST-2 from the GLUE benchmark.
We use BERT-Base-Uncased as the base model for fine-tuning.
Table \ref{tab:bound_measurement} shows the result. 
We find that our bound is orders of magnitude smaller than previous results.
These results show that our theory leads to practical bounds for real-world settings.

\begin{figure*}[!t]
    \begin{minipage}[t]{0.42\textwidth}
        \caption{Comparison of the Hessian-based distance measure (cf. Eq. \eqref{eq_main}) between existing and our fine-tuning approaches.}
        \label{fig:reduced_hessians}
        \centering
        \includegraphics[width=0.85\textwidth]{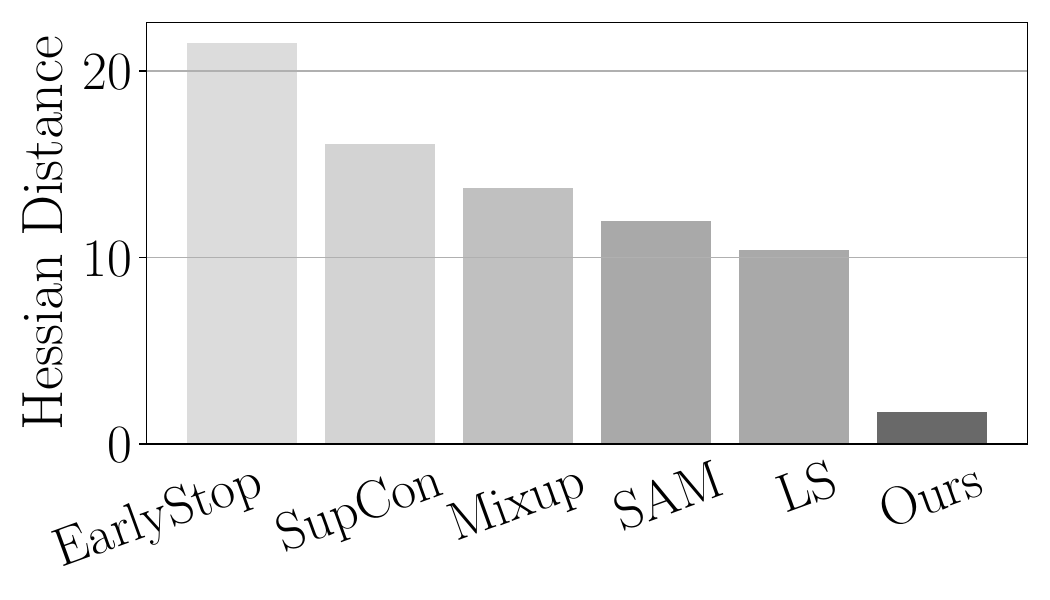}
    \end{minipage}\hfill
    \begin{minipage}[t]{0.56\textwidth}
       \captionof{table}{Numerical comparison between our and known generalization bounds measured with ResNet-50 and BERT-Base as the base model.}\label{tab:bound_measurement}
        \vspace{0.05in}
        \centering
        {\scriptsize\begin{tabular}{@{}lcccc@{}}
        \toprule
        {Generalization bound} & CIFAR-10 & CIFAR-100 & MRPC & SST2 \\
        \midrule
        \citet{pitas2017pac} & 5.51E+10 & 3.13E+12 & / & / \\
        \citet{arora2018stronger} & 1.62E+06 &	9.66E+07 & / & / \\
        \citet{long2020generalization} & 6.30E+13 & 8.32E+13 & / & / \\
        \citet{gouk2020distance} & 2.04E+69 & 2.72E+69 & / & / \\
        \citet{li2021improved} & 1.63E+27 & 1.31E+29 & / & / \\
        \citet{neyshabur2017pac} & 3.13E+11 & 1.62E+13 & / & / \\
        \textbf{Our result} & \textbf{2.26} & \textbf{7.23} & \textbf{3.83} & \textbf{9.71} \\
        \bottomrule
        \end{tabular}}
    \end{minipage}
\end{figure*}

\section{Conclusion}\label{sec_conclude}

This work approached generalization for fine-tuning deep networks using a PAC-Bayesian analysis.
The analysis shows that besides distance from initialization, Hessian plays a crucial role in measuring generalization.
Our theoretical contribution involves Hessian distance-based generalization bounds for various fine-tuning methods on deep networks.
We empirically show that the Hessian distance measure accurately correlates with the generalization error of fine-tuning in practice.
Although we state our results for fine-tuned models, it is worth mentioning that the same statement holds in the supervised learning setting.
In this case, the distance vector $v_i$ in statement \eqref{eq_main} corresponds to the difference between the random initialization of the network and the final trained weight matrix.
Besides, our theoretical result suggests a principled algorithm for fine-tuning with noisy labels.
Experiments demonstrate its robustness under synthetic and real-world label noise.

We describe two questions for future work.
First, our result requires the activation mappings of the neural network to be twice-differentiable. It would be interesting to extend our result to non-smooth activation functions such as ReLU networks (with a loss function such as the cross-entropy loss). For example, in the practical usage of ReLU activation, it is conceivable that non-differentiable states rarely occur. One way is to use smoothed analysis and argue that such non-differentiable states occur with negligible probability.
Second, it would be interesting to apply Hessian to understand generalization in other black-box models.
In a follow-up paper, we apply our techniques to improve the state-of-the-art generalization bounds for graph neural networks \cite{ju2023generalization}.
More broadly, we hope our work inspires further study on the interplay between second-order gradients and generalization.

\section*{Acknowledgment}

Thanks to the anonymous referees for their constructive feedback.
Thanks to Rajmohan Rajaraman, Jason Lee, Aaron Sidford, and Zhiyuan Li for the helpful discussions.
Thanks to Pratik Chaudhari for pointing out several references from the earlier literature.
H. J. and D. L. acknowledge the support of a startup fund and a seed/proof-of-concept grant from the Khoury College of Computer Sciences, Northeastern University.

\begin{refcontext}[sorting=nyt]
\printbibliography
\end{refcontext}

\appendix
\section{Proofs}\label{sec_proofs}

We will analyze the generalization performance of fine-tuning using PAC-Bayesian analysis.
One chooses a \textit{prior} distribution, denoted as $\cP$, and a \textit{posterior} distribution, denoted as $\cQ$.
The \textit{perturbed} empirical loss of $f_W$, denoted by $\hat{\cL}_{\cQ}(f_W)$, is the average of  $\ell_{\cQ}(f_U(x), y)$ over the training dataset.
The \textit{perturbed} expected loss of $f_W$, denoted by ${\cL}_{\cQ}(f_W)$, is the expectation of ${\ell_{\cQ}(f_W(x), y)}$, for a sample $x$ drawn from $\cD$ with class label $y$.
We will use the following PAC-Bayes bound from Theorem 2, \citet{mcallester2013pac} (see also Theorem 2.1, \citet{dziugaite2021role}).

\begin{theorem}\label{lemma_pac}
    Suppose the loss function $\ell(x,y)$ lies in a bounded range $[0, C]$ given any $x\in\cX$ with label $y$.
    For any $\beta\in (0,1)$ and $\delta\in (0,1]$, with probability at least $1-\delta$, the following holds
    \begin{equation}
        \cL_{\cQ}(f_W)\leq \frac{1}{\beta}\hat{\cL}_{\cQ}(f_W) + \frac{C  \Big(KL(\cQ||\cP)+\log\frac{1}{\delta}\Big)}{2\beta(1-\beta)n}. \label{eq_pac}
    \end{equation}
\end{theorem}

This result provides flexibility in setting $\beta$.
Our results will set $\beta$ to balance the perturbation error of $\cQ$ and the KL divergence between $\cP$ and $\cQ$.
We will need the KL divergence between the prior $\cP$ and the posterior $\cQ$ in the PAC-Bayesian analysis. This is stated in the following result.

\begin{proposition}\label{prop_kl}
    Suppose the noise perturbation at layer $i$ is drawn from a Gaussian distribution with mean zero and covariance $\Sigma_i$, for every $i=1,\dots, L$.
    Then, the KL divergence between $\cP$ and $\cQ$ is equal to
    \begin{align*}
        KL(\cQ||\cP) = \frac{1}{2}\sum_{i=1}^L\vect{W_i-\hat{W}_i^{(s)}}^\top\Sigma_i^{-1}\vect{W_i - \hat{W}_i^{(s)}}.
    \end{align*}
    As a corollary, if the population covariance of the noise distribution at every layer is isotropic, i.e., $\Sigma_i = \sigma_i^2\id$, for any $i = 1,\dots, L$, then the KL divergence between $\cP$ and $\cQ$ is equal to
    \begin{align*}
        KL(\cQ||\cP) = \sum_{i=1}^L\frac{\|W_i-\hat{W}_i^{(s)}\|_F^2}{2\sigma^2_i}.
    \end{align*}
\end{proposition}

\begin{proof}
    The proof is based on the definition of multivariate normal distributions.
    Let $Z_i$ be the weight matrix of every layer $i$ in the posterior distribution, for $i$ from $1$ to $L$.
    By the definition of the KL divergence between two distributions, we have 
    \begin{align*}
        & KL(\cQ||\cP)
        = \exarg{Z\sim \cQ}{\log\Big(\frac{\cQ(Z)}{\cP(Z)}\Big)} = \exarg{Z\sim\cQ}{\log \cQ(Z) - \log \cP(Z)} \\
        =& \exarg{Z\sim \cQ}{\sum_{i=1}^L-\frac{1}{2}\vect{Z_i - W_i}^\top\Sigma_i^{-1}\vect{Z_i-W_i} + \frac{1}{2}\vect{Z_i-\hat{W}_i^{(s)}}^\top\Sigma_i^{-1} \vect{Z_i-\hat{W}_i^{(s)}} } \\
        =& -\frac 1 2\exarg{Z\sim \cQ}{\sum_{i=1}^L \bigtr{\vect{Z_i-W_i}\vect{Z_i-W_i}^\top\Sigma_i^{-1}} - \bigtr{ \vect{Z_i-\hat{W}_i^{(s)}}\vect{Z_i-\hat{W}_i^{(s)}}^\top\Sigma_i^{-1}}}.
    \end{align*}
    In the above equation, we recall that the expectation of $Z_i$ is $W_i$.
    Additionally, the population covariance of $Z_i$ is $\Sigma_i$.
    Therefore, after canceling out common terms, we get
    \begin{align*}
        KL(\cQ || \cP) = \sum_{i=1}^L \frac{1}{2} \vect{W_i-\hat{W}_i^{(s)}}^\top\Sigma_i^{-1} \vect{W_i-\hat{W}_i^{(s)}}.
    \end{align*}
    The proof is complete.
\end{proof}

As stated in Theorem \ref{theorem_hessian}, we require the activation functions $\phi_i(\cdot)$ (for any $1\leq i\leq L$) and loss function $\ell(\cdot,\cdot)$ to be twice-differentiable.
Additionally, we require that they are Lipschitz-continuous.
For brevity, we restate these conditions below.

\begin{assumption}\label{assumption_1}
    Assume that the activation functions $\phi_i(\cdot)$ (for any $1\leq i\leq L$) and the loss function $\ell(\cdot, \cdot)$ over the first argument are twice-differentiable and $\kappa_0$-Lipschitz.
    Their first-order derivatives are $\kappa_1$-Lipschitz and their second-order derivatives are $\kappa_2$-Lipschitz.
\end{assumption}

We give several examples of these Lipschitz constants for commonly used activation functions.

\begin{example}
For the Sigmoid function $S(x)$, we have $S(x)\in (0,1)$ for any $x\in\real$. Since the derivative of the Sigmoid function $S'(x) = S(x)(1 - S(x))$, we get that $\kappa_0$, $\kappa_1$, and $\kappa_2$ are at most 1/4.

\smallskip
For the Tanh function $\text{Tanh}(x)$, we have $\text{Tanh}(x)\in (-1,1)$ for any $x\in\real$. Since the derivative of the Tanh function satisfies $\text{Tanh}'(x) = 1 - (\text{Tanh}(x))^2$, we get that $\kappa_0$, $\kappa_1$, and $\kappa_2$ are at most 1.

\smallskip
For the GELU function $\text{GELU}(x)$, we have $\text{GELU}(x) = x\cdot F(x)$ where $F(x)$ be the cumulative density function (CDF) of the standard normal distribution. Let $p(x)$ be the probability density function (PDF) of the standard normal distribution. We have that $F'(x) = p(x)$ and $p'(x) = -xp(x)$. Then, we can get that $\kappa_0$, $\kappa_1$, and $\kappa_2$ at most $1+e^{-1/2}/\sqrt{2\pi}\approx 1.242$.
\end{example}

\subsection{Proof of Theorem \ref{theorem_hessian}}\label{sec_hessian_proof}

\begin{proof}[Proof of Theorem \ref{theorem_hessian}]
    First, we separate the gap of $\cL(f_{\hat W})$ and $\frac{1}{\beta}\hat{\cL}(f_{\hat W})$ into three parts:
    \begin{align*}
        \cL(f_{\hat W}) - \frac{1}{\beta} \hat{\cL}(f_{\hat W})
        = \cL(f_{\hat W}) - \cL_\cQ(f_{\hat W}) + \cL_\cQ(f_{\hat W}) - \frac{1}{\beta}\hat{\cL}_\cQ(f_{\hat W}) + \frac{1}{\beta}\hat{\cL}_\cQ(f_{\hat W}) - \frac{1}{\beta}\hat{\cL}(f_{\hat W}).
    \end{align*}
    By Taylor's expansion of Lemma \ref{lemma_perturbation}, we can bound the noise stability of $f_{\hat W}(x)$ with respect to the empirical loss and the expected loss:
    \begin{align*}
        \cL(f_{\hat W}) - \frac{1}{\beta} \hat{\cL}(f_{\hat W})
        \leq& \Big(-\exarg{(x,y)\sim\cD}{\sum_{i=1}^L\Big\langle \Sigma_i,\mathbf{H}_i[\ell(f_{\hat W}(x),y)]\Big\rangle} + \sum_{i=1}^LC_1\bignormFro{\Sigma_i}^{3/2} \Big) \\
        &+ \Big(\cL_\cQ(f_{\hat W}) - \frac{1}{\beta}\hat{\cL}_\cQ(f_{\hat W})\Big) \\
        &+ \frac{1}{\beta}\Big(\frac{1}{n}\sum_{i=1}^L\sum_{j=1}^n\Big\langle \Sigma_i,\mathbf{H}_i[\ell(f_{\hat W}(x_j),y_j)]\Big\rangle
        + \sum_{i=1}^LC_1\bignormFro{\Sigma_i}^{3/2}\Big),
    \end{align*}
    which is equivalent to the following equation:
    \begin{align}
        \cL(f_{\hat W}) - \frac{1}{\beta} \hat{\cL}(f_{\hat W}) \leq& \Big( -\exarg{(x,y)\sim\cD}{\sum_{i=1}^L\Big\langle \Sigma_i,\mathbf{H}_i[\ell(f_{\hat W}(x),y)]\Big\rangle} \nonumber \\
        & + \frac{1}{n\beta}\sum_{i=1}^L\sum_{j=1}^n\Big\langle \Sigma_i,\mathbf{H}_i[\ell(f_{\hat W}(x_j),y_j)]\Big\rangle \Big) \nonumber \\
        &  + \Big(\frac{1}{\beta} + 1 \Big)\sum_{i=1}^LC_1\bignormFro{\Sigma_i}^{3/2} + \Big(\cL_\cQ(f_{\hat W}) - \frac{1}{\beta}\hat{\cL}_\cQ(f_{\hat W})\Big). \label{eq_combine_3_hess}
    \end{align}
    Next, we combine the upper bound of the noise stability of $f_{\hat W}(x)$ with respect to the empirical loss and the expected loss:
    \begin{align}
        & \frac{1}{n\beta}\sum_{i=1}^L\sum_{j=1}^n\Big\langle \Sigma_i,\mathbf{H}_i[\ell(f_{\hat W}(x_j),y_j)]\Big\rangle
        -\exarg{(x,y)\sim\cD}{\sum_{i=1}^L\Big\langle \Sigma_i,\mathbf{H}_i[\ell(f_{\hat W}(x),y)]\Big\rangle} \nonumber \\
        =& \frac{1}{\beta}\sum_{i=1}^L\Bigg(\frac{1}{n}\sum_{j=1}^n\langle \Sigma_i,\mathbf{H}_i[\ell(f_{\hat W}(x_j),y_j)]\Big\rangle - \exarg{(x,y)\sim\cD}{\langle \Sigma_i,\mathbf{H}_i[\ell(f_{\hat W}(x),y)]\Big\rangle}\Bigg) \label{eq_chernoff_trace_hess} \\
        & +\Big(\frac 1{\beta} - 1\Big)\sum_{i=1}^L\Big\langle \Sigma_i,\exarg{(x,y)\sim\cD}{\mathbf{H}_i[\ell(f_{\hat W}(x_j),y_j)]}\Big\rangle. \nonumber
    \end{align}
    We use the uniform convergence result of Lemma \ref{lemma_union_bound} to bound equation \eqref{eq_chernoff_trace_hess}, leading to:
    \begin{align}
        &\sum_{i=1}^L\Big\langle\Sigma_i, \frac{1}{n}\sum_{j=1}^n\mathbf{H}_i(\ell[f_{\hat{W}}(x_j),y_j)] - \exarg{(x,y)\sim\cD}{\mathbf{H}_i[\ell(f_{\hat{W}}(x),y)]}\Big\rangle \nonumber \\
        \leq& \sum_{i=1}^L\bignormFro{\Sigma_i}\bignormFro{\frac 1 n \sum_{j=1}^n\bH_i[\ell(f_{\hat{W}}(x_j), y_j)] - \exarg{(x,y)\sim\cD}{\bH_i[\ell(f_{\hat{W}}(x), y)]}} \nonumber \\
        \le& \frac{C_2\sqrt{\log(C_3 n/\delta)}}{\sqrt n}\sum_{i=1}^L\bignormFro{\Sigma_i}.\label{eq_con_err_hess}
    \end{align}
    Recall that $v_i$ is a flatten vector of the matrix $\hat{W}_i - \hat{W}_i^{(s)}$. By the PAC-Bayes bound of Theorem \ref{lemma_pac} and the KL divergence result of Proposition \ref{prop_kl},
    \begin{align}
        \cL_\cQ(f_{\hat W}) - \frac{1}{\beta}\hat{\cL}_\cQ(f_{\hat W})
        & \le \frac{C \big(KL(\cQ || \cP) + \log\frac{1}{\delta} \big)}{2\beta(1 - \beta) n} \nonumber\\
        & \le \frac{C \Big(\frac{1}{2}\sum_{i=1}^L\langle v_i, \Sigma_i^{-1} v_i \rangle+\log\frac{1}{\delta} \Big)} {2\beta (1 - \beta) n }. \label{eq_combine_1_hess}
    \end{align}
    Combining equations \eqref{eq_combine_3_hess}, \eqref{eq_con_err_hess}, \eqref{eq_combine_1_hess} with probability at least $1-2\delta$, we get
    \begin{align*}
        \cL(f_{\hat W}) - \frac{1}{\beta}\hat{\cL}(f_{\hat W}) \leq & ~\frac{C_2\sqrt{\log(C_3 n/\delta)}}{\beta\sqrt{n}}\sum_{i=1}^L\bignormFro{\Sigma_i} \\
        & + \Big(\frac 1{\beta} - 1\Big)\sum_{i=1}^L\Big\langle \Sigma_i,\exarg{(x,y)\sim\cD}{\mathbf{H}_i[\ell(f_{\hat W}(x_j),y_j)]}\Big\rangle \\
        & + \Big(\frac{1}{\beta} + 1\Big)C_1\sum_{j=1}^L\bignormFro{\Sigma_i}^{3/2} + \frac{C(\frac{1}{2}\sum_{i=1}^L\langle v_i, \Sigma_i^{-1} v_i\rangle + \log\frac{1}{\delta})}{2\beta(1-\beta)n}.
    \end{align*}
    Recall the truncated Hessian  $\bH_i^+[\ell(f_{\hat W}(x),y)]$ is equal to $U_i \max(D_i, 0) U_i^T$, where $U_iD_iU_i^T$ is the eigen-decomposition of $\bH_i[\ell(f_{\hat W}(x),y)]$. Notice that
    \begin{align*}
       \big\langle \Sigma_i,\mathbf{H}_i[\ell(f_{\hat W}(x), y)]\big\rangle
       \le \big\langle \Sigma_i, \bH_i^+[\ell(f_{\hat W}(x), y)]\big\rangle,
    \end{align*}
    for any $(x, y) \sim \cD$.
    Thus, the above inequality holds after taking an expectation over $(x, y) \sim \cD$ on both sides. 
    Next, the upper bound of $\cL(f_{\hat W}) - \frac{1}{\beta}\hat{\cL}(f_{\hat W})$ relies on $\Sigma_i$ and $\beta$. We aim to select $\Sigma_i$ and $\beta\in (0,1)$ so that this quantity is minimized. For any $x\in\cX$, let $\Sigma_i$ be the matrix that satisfies: %
    \begin{align}
        \Big(\frac 1{\beta} - 1\Big)\big\langle \Sigma_i, { \exarg{(x, y)\sim\cD}{\bH_i^+[\ell(f_{\hat W}(x), y)]}}\big\rangle = \frac{C\langle v_i, \Sigma_i^{-1} v_i\rangle}{4\beta(1-\beta)n}, \label{eq_equal_hess}
    \end{align}
    which implies
    \begin{align}
        \Sigma_i = \sqrt{\frac{C}{4(1-\beta)^2 n \|v_i\|^2}}\exarg{(x, y)\sim\cD}{\bH_i^+[\ell(f_{\hat W}(x), y)]}^{-\frac 1 2} v_i v_i^\top. \label{eq_hess_sigma_1}
    \end{align}
    If the truncated Hessian is zero, then $\Sigma_i$ is set as zero.
    Next, we substitute $\Sigma_i$ into Eq. \eqref{eq_equal_hess} to get
    \begin{align*}
        \Big(\frac{1}{\beta} - 1\Big)\big\langle \Sigma_i,\exarg{(x,y)\sim\cD}{\bH_i^+(\ell(f_{\hat W}(x),y))}\big\rangle
        =& \frac{C\langle v_i,\Sigma_i^{-1} v_i\rangle}{4\beta(1-\beta)n} \\
        =& \sqrt{\frac{C}{4\beta^2 n \| v_i\|^2}}\Big\langle\exarg{(x,y)\sim\cD}{\bH_i^+[\ell(f_{\hat W}(x),y)]}^{\frac 1 2}, v_i v_i^\top\Big\rangle \\
        \leq& \sqrt{\frac{C}{4\beta^2 n \|v_i\|^2}} \bignorm{\exarg{(x,y)\sim\cD}{\bH_i^+[\ell(f_{\hat W}(x),y)]}^{\frac 1 2} v_i} \cdot \norm{v_i} \\
        =& \sqrt{\frac{C\cdot v_i^\top\exarg{(x,y)\sim\cD}{\bH_i^+[\ell(f_{\hat W}(x),y)]} v_i}{4\beta^2 n }} \le \sqrt{\frac{C\cH_i}{4\beta^2 n}}.
    \end{align*}
    Thus, we have shown that equation \eqref{eq_equal_hess} is less than $\sqrt{\frac{C\cH_i}{4\beta^2 n}}$. Next, the gap between $\cL(f_{\hat W})$ and $\frac{1}{\beta}\hat{\cL}(f_{\hat W})$ is
    \begin{align}
        \cL(f_{\hat W})  \leq& \Bigg(\frac{1}{\beta}\hat{\cL}(f_{\hat W}) + \sum_{i=1}^L\sqrt{\frac{C\cH_i}{\beta^2 n }}\Bigg) \nonumber \\
        &+ \Bigg(\frac{C_2\sqrt{\log(C_3 n /\delta)}}{\beta\sqrt{n}}\sum_{i=1}^L \bignormFro{\Sigma_i} + \Big(1+\frac{1}{\beta}\Big)C_1\sum_{i=1}^L\bignormFro{\Sigma_i}^{3/2} + \frac{C}{2\beta(1 - \beta)n}\log\frac{1}{\delta}\Bigg)\label{eq_bigo_1_hess}.
    \end{align}
    Based on equation \eqref{eq_bigo_1_hess}, we set $\epsilon = (1 - \beta)/\beta$. %
    Then we have
    \begin{align*}
        \cL(f_{\hat W}) \leq& (1+\epsilon)\Big(\hat{\cL}(f_{\hat W}) +  \sum_{i=1}^L\sqrt{\frac{C\cH_i}{n}}\Big) \\
        &+ \Bigg( \frac{C_2\sqrt{\log(C_3 n/\delta)}}{\beta\sqrt{n}}\sum_{i=1}^L\bignormFro{\Sigma_i} +  \Big(1 + \frac{1}{\beta}\Big) C_1\sum_{i=1}^L\bignormFro{\Sigma_i}^{3/2} + \frac{C}{2\beta(1-\beta)n}\log\frac{1}{\delta}\Bigg). 
    \end{align*}
    Let $\xi$ be defined as follows:
    \begin{align*}
        \xi = \frac{C_2\sqrt{\log(C_3 n/\delta)}}{\beta\sqrt{n}}\sum_{i=1}^L\bignormFro{\Sigma_i} + \Big(1 + \frac{1}{\beta}\Big)C_1\sum_{i=1}^L\bignormFro{\Sigma_i}^{3/2} + \frac{C}{2\beta(1 - \beta)n}\log\frac{1}{\delta} = O\Big(n^{-3/4}\Big),
    \end{align*}
    where we use the fact that $\bignormFro{\Sigma_i} = \bigo{n^{- 1/2}}$ from equation \eqref{eq_hess_sigma_1}. 
    Hence, we conclude that:
    \begin{align*}
        \cL(f_{\hat W}) \leq (1+\epsilon)\hat{\cL}(f_{\hat W}) +  (1 + \epsilon)\sum_{i=1}^L\sqrt{\frac{C\cH_i}{n}} + \xi. 
    \end{align*}
    Thus, we have shown that equation \eqref{eq_main} holds.
    The proof is now finished.
\end{proof}

\paragraph{Remark.}
This theorem resembles known PAC-Bayes bounds for Lipschitz functions.\footnote{See, e.g., Example 4.55 (Lipschitz functions over a norm-bounded parameter space) of lecture notes by John Duchi: \url{https://web.stanford.edu/class/stats311/lecture-notes.pdf}.}
The analysis crucially relies on the condition that the activation functions and their second derivatives are Lipschitz-continuous.
The crux of this analysis is showing the layerwise Hessian decomposition of the generalization gap.
Notice that similar Hessian measures have appeared in several prior and concurrent works \cite{tsuzuku2020normalized,damian2021label,yang2022does}.
Mathematically formulating these ideas requires dealing with several technical challenges, as discussed in Section \ref{sec_proof_overview}.

\subsection{Proof of Theorem \ref{theorem_noisy}}\label{sec_proof_noisy}

\begin{proof}[Proof of Theorem \ref{theorem_noisy}]
    By equation \eqref{eq_unbiased}, we have the expected loss of $f_{\hat W}$ using $\bar{\ell}$ as
    \[ \cL(f_{\hat W}) = \exarg{(x, y)\sim\cD}{\bar{\ell}(f_{\hat W}(x), y)} = \exarg{(x, y)\sim\cD,\tilde y|y}{\bar{\ell}(f_{\hat W}(x), \tilde y)}. \]
    Denote the empirical loss of $f_{\hat W}$ using $\bar \ell$ as
    \[ \bar{\cL}(f_{\hat W}) = \frac 1 {n} \sum_{i=1}^n \bar{\ell}(f_{\hat W}(x_i), \tilde y_i). \]
    We separate the gap between $\cL\big(f_{\hat W}\big)$ and $\frac{1}{\beta}\bar{\cL}\big(f_{\hat W}\big)$ into three parts:
    \begin{align*}
        \cL\big(f_{\hat W}\big) - \frac{1}{\beta}\bar{\cL}\big(f_{\hat W}\big) \leq \cL\big(f_{\hat W}\big) - \cL_\cQ\big(f_{\hat W}\big) + \cL_\cQ\big(f_{\hat W}\big) - \frac 1 {\beta} \bar{\cL}_\cQ\big(f_{\hat W}\big) + \frac 1 {\beta} \bar{\cL}_\cQ\big(f_{\hat W}\big) - \frac 1 {\beta} \bar{\cL}\big(f_{\hat W}\big).
    \end{align*}
    We denote $\bar{\ell}_\cQ(f_{\hat W}(x), \tilde y)$ as an expectation of $\bar{\ell}(f_{\hat W}(x), \tilde y)$ under the posterior $\cQ$.
    Define $\cL_\cQ\big(f_{\hat W}\big)$ as $\exarg{(x, y)\sim\cD,\tilde y|y}{\bar{\ell}_\cQ(f_{\hat W}(x), \tilde y)}$.
    Define $\bar{\cL}_\cQ(f_{\hat W})$ as $\frac 1 {n} \sum_{i=1}^n \bar{\ell}_\cQ(f_{\hat W}(x_i), \tilde y_i)$.
    By equation \eqref{eq_unbiased}, we have
    \begin{align*}
        \cL\big(f_{\hat W}\big) - \cL_\cQ\big(f_{\hat W}\big) &= \exarg{(x, y)\sim\cD,\tilde y|y}{\bar{\ell}_\cQ(f_{\hat W}(x), \tilde y) - \bar{\ell}(f_{\hat W}(x), \tilde y)}\\
        &= \exarg{(x,y)\sim\cD}{\exarg{\tilde y|y}{\bar{\ell}_\cQ(f_{\hat W}(x), \tilde y) - \bar{\ell}(f_{\hat W}(x), \tilde y)}}\\
        &= \exarg{(x,y)\sim\cD}{\ell_\cQ(f_{\hat W}(x), y) - \ell(f_{\hat W}(x), y)}.
    \end{align*}
    Then, by equation \eqref{eq_weighted_loss}, the noise stability of $f_{\hat W}$ with respect to the empirical loss is equal to
    \begin{align*}
        \frac 1 {\beta} \bar{\cL}_\cQ\big(f_{\hat W}\big) - \frac 1 {\beta} \bar{\cL}\big(f_{\hat W}\big) =& \frac 1 {n\beta} \sum_{j=1}^n \bar{\ell}_\cQ(f_{\hat W}(x_j), \tilde y_j)- \frac 1 {n\beta} \sum_{j=1}^n \bar{\ell}(f_{\hat W}(x_j), \tilde y_j) \\
        =& \frac 1 {n\beta} \sum_{l=1}^k \Lambda_{\tilde y, l}\Bigg(\sum_{j=1}^n \big(\ell_\cQ(f_{\hat W}(x_j), l) - \ell(f_{\hat W}(x_j), l)\big)\Bigg) .
    \end{align*}
    By Taylor's expansion of Lemma \ref{lemma_perturbation}, we can bound the noise stability of $f_{\hat W}$ for the empirical loss and the expected loss:
    {\small\begin{align*}
        \cL\big(f_{\hat W}\big) - \frac{1}{\beta}\bar{\cL}\big(f_{\hat W}\big) 
        \leq& \Bigg(-\exarg{(x,y)\sim\cD}{\sum_{i=1}^L\sigma_i^2\tr\big[\mathbf{H}_i[\ell( f_{\hat W}(x),y)]\big]} + \sum_{i=1}^L C_1\sigma_i^3\Bigg) + \Big(\cL_\cQ\big(f_{\hat W}\big) - \frac{1}{\beta}\bar{\cL}_\cQ\big(f_{\hat W}\big)\Big) \\
        & + \frac{1}{\beta}\Bigg(\frac{1}{n}\sum_{i=1}^L\sigma_i^2\Big(\sum_{l=1}^k \Lambda_{\tilde y, l}\big(\sum_{j=1}^n\bigtr{\mathbf{H}_i[\ell( f_{\hat W}(x_j),l)]}\big)\Big) +  \sum_{l=1}^k\Lambda_{\tilde y, l}\sum_{i=1}^LC_1\sigma_i^3 \Bigg), 
    \end{align*}}%
    which is equivalent to the following equation:
    {\small\begin{align}
        & \cL\big(f_{\hat W}\big) - \frac{1}{\beta}\bar{\cL}\big(f_{\hat W}\big) \\
        \leq& \Bigg(-\exarg{(x,y)\sim\cD}{\sum_{i=1}^L\sigma_i^2\tr\big[\mathbf{H}_i[\ell( f_{\hat W}(x),y)]\big]} + \frac{1}{n\beta}\sum_{i=1}^L\sigma_i^2\Big(\sum_{l=1}^k \Lambda_{\tilde y, l}\big(\sum_{j=1}^n\bigtr{\mathbf{H}_i[\ell( f_{\hat W}(x_j),l)]}\big)\Big)\Bigg) \nonumber \\
        & + \Bigg( \sum_{i=1}^L C_1\sigma_i^3 + \frac{1}{\beta} \sum_{l=1}^k\Lambda_{\tilde y, l}\sum_{i=1}^LC_1\sigma_i^3 \Bigg) + \Big(\cL_\cQ\big(f_{\hat W}\big) - \frac{1}{\beta}\bar{\cL}_\cQ\big(f_{\hat W}\big)\Big). \label{eq_combine_3_noisy}
    \end{align}}%
    Next, we combine the upper bound of the noise stability of $f_{\hat W}(x)$ for the empirical loss and the expected loss:
    \begin{align}
        &\frac{1}{n\beta}\sum_{i=1}^L\sigma_i^2\Big(\sum_{l=1}^k \Lambda_{\tilde y, l}\big(\sum_{j=1}^n\bigtr{\mathbf{H}_i[\ell( f_{\hat W}(x_j),l)]}\big)\Big) - \exarg{(x,y)\sim\cD}{\sum_{i=1}^L\sigma_i^2\tr\big[\mathbf{H}_i[\ell( f_{\hat W}(x),y)]\big]} \nonumber \\
        =& \sum_{i=1}^L\sigma_i^2\Bigg(\sum_{l=1}^k\Lambda_{\tilde y, l}\Big(\frac{1}{n}\sum_{j=1}^n\bigtr{\mathbf{H}_i[\ell( f_{\hat W}(x_j),l)]} - \exarg{(x,y)\sim\cD}{\tr\big[\mathbf{H}_i[\ell( f_{\hat W}(x),y)]\big]} \Big)\Bigg) \label{eq_chernoff_trace_noisy} \\
        &+ \Big(\frac{1}{\beta} - 1\Big)\sum_{i=1}^L\sigma_i^2\Bigg(\sum_{l=1}^k \Lambda_{\tilde y, l}\Big(\frac{1}{n}\sum_{j=1}^n\bigtr{\mathbf{H}_i[\ell(f_{\hat W}(x_j),l)]}\Big)\Bigg). \label{eq_chernoff_trace_noisy_1}
    \end{align}
    We use the uniform convergence result of Lemma \ref{lemma_union_bound} to bound equation \eqref{eq_chernoff_trace_noisy}, leading to a concentration error term of
    \begin{align}
        \sum_{i=1}^L\sigma_i^2\max_{\tilde y}\sum_{j=1}^L\bigabs{\Lambda_{\tilde y, j}}\frac{C_2\sqrt{\log(C_3 n/\delta)}}{\sqrt{n}} 
        = \frac{C_2\sqrt{\log(C_3 n/\delta)}}{\sqrt{n}}\sum_{i=1}^L\sigma_i^2\rho. \label{eq_con_err_noisy}
    \end{align}
    For equation \eqref{eq_chernoff_trace_noisy_1}, the upper bound will be
    {\begin{align}
        & \Big(\frac{1}{\beta} - 1\Big)\sum_{i=1}^L\sigma_i^2\Bigg(\sum_{l=1}^k \Lambda_{\tilde y, l}\Big(\frac{1}{n}\sum_{j=1}^n\bigtr{\mathbf{H}_i[\ell(f_{\hat W}(x_j),l)]}\Big)\Bigg) \nonumber \\
        \leq & \Big(\frac{1}{\beta} - 1\Big)\sum_{i=1}^L\sigma_i^2\rho\max_{x\in\cX, y\in\set{1,\dots,k}} \abs{\tr[\bH_i(\ell(f_{\hat W}(x), y))]}, \label{eq_combine_2_noisy}
    \end{align}}%
    where $\rho = \norm{\Lambda^{\top}}_{1,\infty}$.
    Recall that $v_i$ is a flatten vector of the matrix $\hat{W_i} - \hat{W_i}^{(s)}$. By the PAC-Bayes bound of Theorem \ref{lemma_pac} and Proposition \ref{prop_kl},
    \begin{align}
        \cL_\cQ\big(f_{\hat W}\big) - \frac{1}{\beta}\bar{\cL}_\cQ\big(f_{\hat W}\big) 
        &\le \frac{C\Big(KL(\cQ || \cP) + \log\frac{1}{\delta}\Big)}{2\beta(1 - \beta) n} \nonumber \\
        &\le \frac{C \Big(\sum_{i=1}^L\frac{\bignorm{v_i}^2}{2\sigma_i^2}+\log\frac{1}{\delta} \Big)} {2\beta (1 - \beta) n }. \label{eq_combine_1_noisy}
    \end{align}
    Combining equations \eqref{eq_combine_3_noisy}, \eqref{eq_combine_2_noisy}, \eqref{eq_combine_1_noisy}, with probability at least $1 - 2\delta$, we get
    \begin{align}
        \cL\big(f_{\hat W}\big) - \frac{1}{\beta}\bar{\cL}\big(f_{\hat W}\big) 
        \leq& \frac{C_2\rho\sqrt{\log(C_3 n/\delta)}}{\sqrt{n}}\sum_{i=1}^L\sigma_i^2 + \Big(\frac {1}{\beta} - 1\Big)\rho\sum_{i=1}^L\sigma_i^2\max_{x\in\cX, y\in\set{1,\dots,k}} \abs{\tr[\bH_i(\ell(f_{\hat W}(x), y))]} \nonumber \\
        & + \Big(\frac{\rho}{\beta} + 1\Big)\sum_{j=1}^LC_1\sigma_i^3 + \frac{C(\sum_{i=1}^L\frac{\bignorm{v_i}^2}{2\sigma_i^2}+\log\frac{1}{\delta})}{2\beta(1-\beta)n}. \label{eq_sigma_noisy}
    \end{align}
    In equation \eqref{eq_sigma_noisy}, the upper bound of $\cL\big(f_{\hat W}\big) - \frac{1}{\beta}\bar{\cL}\big(f_{\hat W}\big)$ relies on $\sigma_i$ and $\beta$. Our goal is to select  $\sigma_i>0$ and $\beta\in (0,1)$ to minimize equation \eqref{eq_sigma_noisy}.
    This is achieved when:
    \begin{align}
        \sigma_i^2 = \sqrt{\frac{C\bignorm{v_i}^2}{4\rho(1-\beta)^2 n\cdot\max_{x\in\cX, y\in\set{1,\dots,k}} \abs{\tr[\bH_i(\ell(f_{\hat W}(x), y))]}}}, \label{eq_noisy_sigma_1}
    \end{align}
    for any $1\leq i\leq L$, following the condition:
    \begin{align*}
        \Big(\frac {1}{\beta} - 1\Big)\rho\sigma_i^2\max_{x\in\cX, y\in\set{1,\dots,k}} \abs{\tr[\bH_i(\ell(f_{\hat W}(x), y))]} = \frac{C}{2\beta(1-\beta)n}\frac{\bignorm{v_i}^2}{2\sigma_i^2}.
    \end{align*}
    Then, equation \eqref{eq_sigma_noisy} becomes
    \begin{align}
        \cL\big(f_{\hat W}\big) \leq& \Bigg(\frac{1}{\beta}\bar{\cL}\big(f_{\hat W}\big) + \sum_{i=1}^L\sqrt{\frac{C\rho \max_{x\in\cX, y\in\set{1,\dots,k}} \abs{\tr[\bH_i(\ell(f_{\hat W}(x), y))]}\cdot\bignorm{v_i}^2}{\beta^2 n}}\Bigg) \\
        &+ \Bigg(\frac{C_2\rho\sqrt{\log(C_3 n/\delta)}}{\sqrt{n}} \sum_{i=1}^L\sigma_i^2 + \frac{C}{2\beta(1-\beta)n}\log\frac{1}{\delta} + \Big(\frac{\rho}{\beta} + 1\Big)C_1\sum_{i=1}^L\sigma_i^3\Bigg). \label{eq_bigo_1_noisy}
    \end{align}
    Set $\beta$ as a fixed value close to 1 that does not grow with $n$ and $1 / \delta$.
    Let $\epsilon = (1 - \beta)/\beta$. We get
    \begin{align*}
        \cL\big(f_{\hat W}\big) \leq& (1 + \epsilon)\Bigg(\bar{\cL}\big(f_{\hat W}\big) + \sum_{i=1}^L\sqrt{\frac{C\rho\max_{x\in\cX, y\in\set{1,\dots,k}} \abs{\tr[\bH_i(\ell(f_{\hat W}(x), y))]}\cdot\bignorm{v_i}^2}{n}}\Bigg) \\
        &+ \Bigg(\frac{C_2\rho\sqrt{\log(C_3 n/\delta)}}{\sqrt{n}}\sum_{i=1}^L\sigma_i^2 + \Big(\frac{\rho}{\beta} + 1\Big)C_1\sum_{i=1}^L\sigma_i^3 + \frac{C}{2\beta(1-\beta)n}\log\frac{1}{\delta}\Bigg).
    \end{align*}
    Let $\xi$ be defined as follows. We have
    \begin{align*}
        \xi = \frac{C_2\rho\sqrt{\log(C_3 n/\delta)}}{\sqrt{n}}\sum_{i=1}^L\sigma_i^2 + \Big(\frac{\rho}{\beta} + 1\Big)C_1\sum_{i=1}^L\sigma_i^3 + \frac{C}{2\beta(1-\beta)n}\log\frac{1}{\delta} = O\Big(n^{-3/4}\Big),
    \end{align*}
    where we notice that $\sigma_i^2 = \bigo{n^{- 1/2}}$ because of equation \eqref{eq_noisy_sigma_1}.
    Thus, we conclude that
    \begin{align*}
        \cL\big(f_{\hat W}\big) \leq (1 + \epsilon)\bar{\cL}\big(f_{\hat W}\big) + (1 + \epsilon)\sum_{i=1}^L\sqrt{\frac{C\rho\max_{x\in\cX, y\in\set{1,\dots,k}} \abs{\tr[\bH_i(\ell(f_{\hat W}(x), y))]}\cdot\bignorm{v_i}^2}{n}} + \xi.
    \end{align*}
    Since $\bignorm{v_i}\le \alpha_i$, for any $i = 1,\dots, L$, we have finished the proof.
\end{proof}

\subsection{Proof of Lemma \ref{lemma_union_bound}}\label{proof_union}

In this section, we provide the proof of Lemma \ref{lemma_union_bound}.
We need a perturbation analysis of the Hessian as follows.
Given an $L$-layer neural network $f_W: \cX \rightarrow \mathbb{R}^{d_L}$ and any input $z^0\in\cX$, let $z_W^{i} =  \phi_i(W_{i} z_W^{i-1})$ be the activated output vector from layer $i$, for any $i = 1,2,\dots,L$.

We show that the Hessian matrix is Lipschitz-continuous. In particular, there exists a fixed value $G$ that does not grow with $n$ and $\delta^{-1}$, such that the changes in the Hessian can be quantified below.

\begin{proposition}\label{prop_hessian_lip}
    Suppose that Assumption \ref{assumption_1} holds.
    The change in the Hessian output of the loss function $\ell(f_W(x),y)$ for $W_i$ under perturbation $U$ can be bounded as follows: 
    \begin{align}
        \bignormFro{\bH_i[\ell(f_{W}(x), y)] \mid_{W_i + U_i, \forall 1\le i\le L} - \bH_i[\ell(f_W(x), y)]} 
        \leq G\normFro{U},
    \end{align}
    where $G$ has the following equation:
    \begin{align}\label{eq_G_constant}
        G = \frac{3}{2}(L+1)^2e^6\kappa_2\kappa_1^2\kappa_0^{3(L+1)}\Big(\max_{x\in\cX}\bignorm{x}\prod_{l=0}^{L}d_l\prod_{h=1}^{L}\bignorms{W_h}\Big)^3\Big(\max_{1\leq i\leq L}\frac{1}{\bignorms{W_i}^2}\sum_{h=1}^{L}\frac{1}{\bignorms{W_h}}\Big).
    \end{align}
    where $\kappa_0,\kappa_1,\kappa_2 \ge 1$ is the maximum over the Lipschitz constants, first-order derivative Lipschitz constants, and second-order derivative Lipschitz constants of all the activation functions and the loss function.
\end{proposition}

Provided with the Lipschitz-continuity conditions, now we are ready to state the proof of Lemma \ref{lemma_union_bound}, which shows the uniform convergence of the loss Hessian.

\begin{proof}[Proof of Lemma \ref{lemma_union_bound}]
    Let $B$, $\epsilon > 0$, $p = \sum_{i=1}^L d_id_{i-1}$, and let $S = \{\text{vec}(W)\in\real^{p}: \bignorms{\text{vec}(W)}\leq B\}$. Then there exists an $\epsilon$-cover of $S$ with respect to the $\ell_2$-norm at most $\max\Big(\big(\frac{3B}{\epsilon}\big)^p,1\Big)$ elements from Example 5.8, \citet{wainwright2019high}. From Proposition \ref{prop_hessian_lip}, the Hessian $\bH_i[\ell(f_W(x),y)]$ is  $G$-Lipschitz for all $(W+U),W\in S$ and any $i = 1,2,\cdots,L$. Then we have
    \begin{align*}
        \bignormFro{\bH_i[\ell(f_{W+U}(x),y)] - \bH_i[\ell(f_W(x),y)]}\leq G\bignormFro{U},
    \end{align*}
    where $G$ is defined in equation \eqref{eq_G_constant}. %
    For parameters $\delta,\epsilon > 0$, let $\cN$ be the $\epsilon$-cover of $S$ with respect to the $\ell_2$-norm. Define the event 
    \begin{align*}
        E=\Big\{\forall W\in \cN, \bignormFro{\frac{1}{n}\sum_{j=1}^n\bH_i[\ell(f_{W}(x_j),y_j)] - \exarg{(x,y)\sim\cD}{ \bH_i[\ell(f_W(x),y)]}}\leq\delta\Big\}.
    \end{align*}
    By the matrix Bernstein inequality, we have $$\text{Pr}(E) \geq 1 - 4|\cN|d_id_{i-1}\exp(-n\delta^2/(2\cH_i^2)).$$
    Next, for any $W\in S$, we can pick some $W+U\in \cN$ such that $\bignormFro{U}\leq \epsilon$. We have
    \begin{align*}
        &\bignormFro{\exarg{(x,y)\sim\cD}{\bH_i[\ell(f_{W+U}(x),y)]} - \exarg{(x,y)\sim\cD}{\bH_i[\ell(f_W(x),y)]}}\leq G\bignormFro{U}\leq G\epsilon, \text{ and} \\
        &\bignormFro{\frac{1}{n}\sum_{j=1}^n\bH_i[\ell(f_{W+U}(x_j),y_j)] -  \frac{1}{n}\sum_{j=1}^n\bH_i[\ell(f_W(x_j),y_j)]}\leq G\bignormFro{U}\leq G\epsilon.
    \end{align*}
    Therefore, for any $W\in S$, we obtain:
    \begin{align*}
        \bignormFro{\frac{1}{n}\sum_{j=1}^n\bH_i[\ell(f_{W}(x_j),y_j)] - \exarg{(x,y)\sim\cD}{ \bH_i[\ell(f_W(x),y)]}}\leq 2G\epsilon + \delta.
    \end{align*}
    We will also set the value of $\delta$ and $\epsilon$. First, set $\epsilon = \delta/(2G)$ so that conditional on $E$,
    \begin{align*}
        \bignormFro{\frac{1}{n}\sum_{j=1}^n\bH_i[\ell(f_{W}(x_j),y_j)] - \exarg{(x,y)\sim\cD}{ \bH_i[\ell(f_W(x),y)]}}\leq 2\delta.
    \end{align*}
    The event $E$ happens with a probability of at least:
    \begin{align*}
        1 - 4|\cN|d_id_{i-1}\exp(-n\delta^2/(2\cH_i^2)) = 1 - 4d_id_{i-1}\exp(\log |\cN| -n\delta^2/(2\cH_i^2)).
    \end{align*}
    We have $\log |\cN| \leq p\log(3B/\epsilon) = p\log (6BG/\delta)$. If we set $$\delta = \sqrt{\frac{4p\cH_i^2\log(3\tau BGn/\cH_i)}{n}}$$ so that $\log(3\tau BGn/\cH_i) \geq 1$ (because $n\geq \frac{e\cH_i}{3BG}$ and $\tau\geq 1$), then we get
    \begin{align*}
        & p\log(6BG/\delta) - n\delta^2/(2\cH_i^2) \\
        =&p\log\bigbrace{\frac{6BG\sqrt{n}}{\sqrt{4p\cH_i^2\log(3\tau BGn/\cH_i)}}} - 2p\log(3\tau BGn/\cH_i) \\
        =&p\log\bigbrace{\frac{3BG\sqrt{n}}{\cH_i\sqrt{p\log(3\tau BGn/\cH_i)}}} - 2p\log(3\tau BGn/\cH_i) \\
        \leq& p\log(3\tau BGn/\cH_i) - 2p\log(3\tau BGn/\cH_i) \tag{$\tau\ge 1,\log(3\tau BGn/\cH_i)\geq 1$} \\
        =& -p\log(3\tau BGn/\cH_i) \\
        \leq& -p\log(e\tau) \tag{$3BGn/\cH_i\geq e$}.
    \end{align*}
    Therefore, with probability greater than $$1 - 4|\cN|d_id_{i-1}\exp(-n\delta^2/(2\cH_i^2))\geq 1 - 4d_id_{i-1}(e\tau)^{-p}\geq 1 - 4d_id_{i-1}/(e\tau),$$
    the following estimate holds:
    \begin{align*}
        \bignormFro{\frac{1}{n}\sum_{j=1}^n\bH_i[\ell(f_{W}(x_j),y_j)] - \exarg{(x,y)\sim\cD}{ \bH_i[\ell(f_W(x),y)]}}\leq \sqrt{\frac{16p\cH_i^2\log(3\tau BGn/\cH_i)}{n}}.
    \end{align*}
    Denote $\delta' = 4d_id_{i-1}/(e\tau)$, $C_2 = 4\cH_i\sqrt{p}$, and $C_3 = 12 d_id_{i-1} BG/(e\cH_i)$. With probability greater than $1 - \delta'$, the final result is:
    \begin{align*}
        \bignormFro{\frac{1}{n}\sum_{j=1}^n\bH_i[\ell(f_{W}(x_j),y_j)] - \exarg{(x,y)\sim\cD}{ \bH_i[\ell(f_W(x),y)]}}\leq C_2\sqrt{\frac{\log(C_3 n/\delta')}{n}}.
    \end{align*}
    This completes the proof of Lemma \ref{lemma_union_bound}.
\end{proof}

\subsection{Proof of Proposition \ref{prop_hessian_lip}}

We divide the proof of Proposition \ref{prop_hessian_lip} into several parts.
Most of the calculation is devoted to getting an explicit value of $G$.
First, we derive the perturbation of the network output.

\begin{claim}\label{claim_lip}
    Suppose that Assumption \ref{assumption_1} holds. 
    For any $j = 1,2,\cdots,L$, the change in the output of the $j$ layer network $z_W^j$ with perturbation $U$ can be bounded as follows: 
    \begin{align}\label{eq_result_lip}
        \bignorm{z_{W+U}^j - z_W^j} \leq \Big(1 + \frac{1}{L}\Big)^j \kappa_0^j\Bigg(\max_{x\in\cX}\bignorm{x}\prod_{h=1}^j\bignorms{W_h}\sum_{h=1}^j\frac{\bignorms{U_h}}{\bignorms{W_h}}\Bigg).
    \end{align}
    where $\kappa_0 \geq 1$ is the maximum over the Lipschitz constants of all the activation functions and the loss function.
\end{claim}

See Lemma 2 in \citet{neyshabur2017pac} for proof. Second, we derive the perturbation of the derivative of the network.

\begin{claim}\label{prop_grad}
    Suppose that Assumption \ref{assumption_1} holds.
    The change in the Jacobian of output of the $j$ layer network $z_W^j$ for $W_i$ under perturbation $U$ can be bounded as follows:
    \begin{align}\label{eq_result_grad}
        \bignormFro{\frac{\partial}{\partial W_i}z_W^j\Big|_{W_k+U_k,\forall 1\leq k\leq j} - \frac{\partial}{\partial W_i}z_W^j} 
        \leq A_{i,j}\Bigg(\max_{x\in\cX}\bignorm{x}\prod_{h=1}^{j}\bignorms{W_h}\sum_{h=1}^j\frac{\bignorms{U_h}}{\bignorms{W_h}}\Bigg).
    \end{align}
    For any $j = 1,2,\cdots, L$,
    $A_{i,j}$ has the following equation:
    \begin{align*}
        A_{i,j} = (j-i+2)\Big(1 + \frac{1}{L}\Big)^{3j}\kappa_1\kappa_0^{2j-1}\Bigg(\max_{x\in\cX}\bignorm{x}\frac{1}{\bignorms{W_i}}\prod_{l=i-1}^{j}d_l\prod_{h=1}^{j}\bignorms{W_h}\Bigg),
    \end{align*}
    where $\kappa_0,\kappa_1 \ge 1$ is the maximum over the Lipschitz constants and first-order derivative Lipschitz constants of all the activation functions and the loss function. %
\end{claim}

\begin{proof}%
    For any fixed $i = 1,2,\cdots, L$, we will prove using induction for $j$. 
    If $j < i$, we have $\frac{\partial}{\partial W_i}z_W^j = 0$ since $z_W^j$ remains unchanged regardless of any change of $W_i$. When $j = i$, we have
    \begin{align*}
        &\bignorm{\frac{\partial}{\partial W_i^{p,q}}z_W^i\Bigg|_{W_k+U_k,\forall 1\leq k\leq i} - \frac{\partial}{\partial W_i^{p,q}}z_W^i} \\
        =& \bignorm{\frac{\partial}{\partial W_i^{p,q}}\phi_i(W_i z_W^{i-1})\Bigg|_{W_k + U_k,\forall 1\leq k\leq i} - \frac{\partial}{\partial W_i^{p,q}}\phi_i(W_i z_W^{i-1})} \\
        =& \bigabs{[z^{i-1}_{W+U}]^q\phi_i'\Bigg(\sum_{r=1}^{d_{i-1}}\Big((W_i^{p,r} + U_i^{p,r}) [z_{W+U}^{i-1}]^r\Big)\Bigg) - [z_W^{i-1}]^q\phi_i'\Big(\sum_{r=1}^{d_{i-1}}W_i^{p,r}[z_W^{i-1}]^r\Big)},
    \end{align*}
    where $W_i^{p,q},U_i^{p,q}$ are the $(p,q)$'th entry of $W_i$ and $U_i$. 

    Let $[z_W^{i-1}]^q$ be the $q$'th entry of $z_W^{i-1}$.  By the triangle inequality, the above equation is smaller than
    \begin{align}
        &\bigabs{[z_{W+U}^{i-1}]^q} \bigabs{\phi_i'\Big(\sum_{r=1}^{d_{i-1}}\Big((W_i^{p,r} + U_i^{p,r}) [z_{W+U}^{i-1}]^r\Big) - \phi_i'\Big(\sum_{r=1}^{d_{i-1}}W_i^{p,r}[z_W^{i-1}]^r\Big)} \label{eq_grad_1}\\
        &+ \bigabs{[z_{W+U}^{i-1}]^q - [z_W^{i-1}]^q}\bigabs{\phi_i'\Big(\sum_{r=1}^{d_{i-1}}W_i^{p,r}[z_W^{i-1}]^r\Big)}. \label{eq_grad_2}
    \end{align}
    Since $\phi_i'$ is $\kappa_1$-Lipschitz by Assumption \ref{assumption_1}, equation \eqref{eq_grad_1} is at most
    \begin{align}\label{eq_grad_7}
        \bigabs{[z_{W+U}^{i-1}]^q} \cdot\kappa_1\bigabs{\sum_{r=1}^{d_{i-1}}(W_i^{p,r} + U_i^{p,r}) [z_{W+U}^{i-1}]^r - \sum_{r=1}^{d_{i-1}}W_i^{p,r}[z_W^{i-1}]^r}.
    \end{align}
    Since Assumption \ref{assumption_1} holds, we have $\bigabs{\phi_i'(\cdot)}\leq \kappa_0$, by the Lipschitz condition on the first-order derivatives.
    Thus, equation \eqref{eq_grad_2} is at most
    \begin{align}\label{eq_grad_8}
        \bigabs{[z_{W+U}^{i-1}]^q - [z_W^{i-1}]^q}\cdot\kappa_0.
    \end{align}

    Taking equation \eqref{eq_grad_7} and \eqref{eq_grad_8} together, the Frobenius norm of the Jacobian is
    \begin{align*}
        &\bignormFro{\frac{\partial}{\partial W_i}z_W^i\Bigg|_{W_k+U_k,\forall 1\leq k\leq i} - \frac{\partial}{\partial W_i}z_W^i} \\
        =~& \sqrt{\sum_{p=1}^{d_i}\sum_{q=1}^{d_{i-1}}\bignorm{\frac{\partial}{\partial W_i^{p,q}}z_W^i\Bigg|_{W_k+U_k,\forall 1\leq k\leq i} - \frac{\partial}{\partial W_i^{p,q}}z_W^i}^2} \\
        \leq~& d_i\bignorm{z_{W+U}^{i-1}}\cdot k_1d_{i-1}\bignorm{\big(W_i+U_i\big)z_{W+U}^{i-1} - W_iz_W^{i-1}} + \kappa_0 d_i \bignorm{z_{W+U}^{i-1} - z_W^{i-1}}.
    \end{align*}

    By the Lipschitz continuity of $z_{W+U}^{i-1}$ and equation \eqref{eq_result_lip} in Claim \ref{claim_lip}, we have
    \begin{align}
        \bignorm{z_{W+U}^{i-1}}&\leq \Big(1 + \frac{1}{L}\Big)^{i-1}\kappa_0^{i-1}\max_{x\in\cX}\bignorm{x}\prod_{h=1}^{i-1}\bignorms{W_h}. \label{eq_grad_rule_1} \\
        \bignorm{z_{W+U}^{i-1} - z_W^{i-1}}&\leq \Big(1 + \frac{1}{L}\Big)^{i-1} \kappa_0^{i-1}\max_{x\in\cX}\bignorm{x}\prod_{h=1}^{i-1}\bignorms{W_h}\sum_{h=1}^{i-1}\frac{\bignorms{U_h}}{\bignorms{W_h}}. \label{eq_grad_rule_2} 
    \end{align}
    By the triangle inequality, we have
    \begin{align}\label{eq_grad_9}
        \bignorm{\big(W_i+U_i\big)z_{W+U}^{i-1} - W_iz_W^{i-1}}\leq \bignorms{W_i+U_i}\bignorm{z_{W+U}^{i-1} - z_W^{i-1}} + \bignorms{U_i}\bignorm{z_W^{i-1}}.
    \end{align}

    From equations \eqref{eq_grad_rule_1}, \eqref{eq_grad_rule_2}, and \eqref{eq_grad_9}, we finally obtain that the Jacobian of $z_W^i$ with respect to $W_i$ is at most
    \begin{align*}
        &\bignormFro{\frac{\partial z_W^i}{\partial W_i}\Bigg|_{W_k+U_k,\forall 1\leq k\leq i} - \frac{\partial z_W^i}{\partial W_i}} \\ 
        \leq& \Bigg(2\Big(1 + \frac{1}{L}\Big)^{2i-1}\kappa_1\kappa_0^{2i-1}d_id_{i-1}\max_{x\in\cX}\bignorm{x}\prod_{h=1}^{i-1}\bignorms{W_h}\Bigg)\Bigg(\max_{x\in\cX}\bignorm{x}\prod_{h=1}^{i}\bignorms{W_h}\sum_{h=1}^i\frac{\bignorms{U_h}}{\bignorms{W_h}}\Bigg).
    \end{align*}
    Hence, we know that equation \eqref{eq_result_lip} will be correct when $j = i$. 
    
    Assuming that equation \eqref{eq_result_lip} will be correct for any $i$ up to $i\le j$, we have
    \begin{align*}
        &\bignormFro{\frac{\partial z_W^j}{\partial W_i}\Bigg|_{W_k+U_k,\forall 1\leq k\leq j} - \frac{\partial z_W^j}{\partial W_i}} \\ 
        \leq& (j-i+2)\Big(1 + \frac{1}{L}\Big)^{3j}\kappa_1\kappa_0^{2j-1} \Bigg(\frac{\max_{x\in\cX}\bignorm{x} }{\bignorms{W_i}}\prod_{l=i-1}^{j}d_l\prod_{h=1}^{j}\bignorms{W_h}\Bigg) \max_{x\in\cX}\bignorm{x}\prod_{h=1}^{j}\bignorms{W_h}\sum_{h=1}^j\frac{\bignorms{U_h}}{\bignorms{W_h}}.
    \end{align*}
    From layer $i$ to layer $j+1$, we show the derivative of $z_W^{j+1} = \phi_{j+1}(W_{j+1}z_W^{j})$ with respect to $W_i^{p,q}$ as follows:
    \begin{align*}
        \frac{\partial \phi_{j+1}(W_{j+1}z_W^{j})}{\partial W_i^{p,q}} = \phi_{j+1}'(W_{j+1}z_W^{j})\odot \Bigg(W_{j+1}\frac{\partial z_W^{j}}{\partial W_i^{p,q}}\Bigg),
    \end{align*}
    where $\odot$ is the Hadamard product operator between two vectors. Then we have 
    \begin{align*}
        &\bignorm{\frac{\partial z_W^{j+1}}{\partial W_i^{p,q}}\Bigg|_{W_k+U_k,\forall 1\leq k\leq j+1} - \frac{\partial z_W^{j+1}}{\partial W_i^{p,q}}} \nonumber \\
        =& \bignorm{\phi_{j+1}'\big((W_{j+1}+U_{j+1})z_{W+U}^{j}\big)\odot (W_{j+1}+U_{j+1})\frac{\partial z_W^{j}}{\partial W_i^{p,q}}\Bigg|_{W_k+U_k} - \phi_{j+1}'(W_{j+1}z_W^{j})\odot W_{j+1}\frac{\partial z_W^{j}}{\partial W_i^{p,q}}}
    \end{align*}

    By the Cauchy-Schwarz inequality and the triangle inequality, the above equation is, at most
    {\small\begin{align}
        & \bignorm{\phi_{j+1}'\big((W_{j+1}+U_{j+1})z_{W+U}^j\big) - \phi_{j+1}'(W_{j+1}z_W^j)}\bignorm{(W_{j+1}+U_{j+1})\frac{\partial z_W^j}{\partial W_i^{p,q}}\Bigg|_{W_k+U_k,\forall 1\leq k\leq j}} \label{eq_grad_3} \\
        &+ \bignorm{\phi_{j+1}'(W_{j+1}z_W^{j})}\bignorm{(W_{j+1}+U_{j+1})\frac{\partial z_W^j}{\partial W_i^{p,q}}\Bigg|_{W_k+U_k,\forall 1\leq k\leq j} - W_{j+1}\frac{\partial z_W^{j}}{\partial W_i^{p,q}}}. \label{eq_grad_4}
    \end{align}}%
    Using the triangle inequality and the $\kappa_1$-Lipschitz of $\phi_{j+1}'(\cdot)$, equation \eqref{eq_grad_3} is at most
    \begin{align*}
        \kappa_1\Big(\bignorms{W_{j+1}+U_{j+1}}\bignorm{z_{W+U}^j - z_W^j} + \bignorms{U_{j+1}}\bignorm{z_W^j}\Big)\cdot \bignorms{W_{j+1}+U_{j+1}}\bignorm{\frac{\partial z_W^j}{\partial W_i^{p,q}}\Bigg|_{W_k+U_k,\forall 1\leq k\leq j}}. 
    \end{align*}
    Using the triangle inequality and $|\phi_{j+1}'(\cdot)|\leq\kappa_0$, equation \eqref{eq_grad_4} is at most
    \begin{align*}
        d_{j+1}\kappa_0\cdot\Bigg(\bignorms{W_{j+1}+U_{j+1}}\bignorm{\frac{\partial z_W^{j}}{\partial W_i^{p,q}}\Bigg|_{W_k+U_k,\forall 1\leq k\leq j} - \frac{\partial z_W^{j}}{\partial W_i^{p,q}}} + \bignorms{U_{j+1}}\bignorm{\frac{\partial z_W^{j}}{\partial W_i^{p,q}}}\Bigg).
    \end{align*}
    From Claim \ref{claim_lip}, we have the norm of the derivative of $z_W^j$ with respect to $W_i^{p,q}$ bounded by its Lipschitz-continuity in equation \eqref{eq_result_lip}:
    \begin{align}
        \bignorm{\frac{\partial z_W^{j}}{\partial W_i^{p,q}}}&\leq (1+\frac{1}{L})^{j}\kappa_0^{j}\max_{x\in\cX}\bignorm{x}\frac{1}{\bignorms{W_i}}\prod_{h=1}^{j}\bignorms{W_h}. \label{eq_grad_rule_3} \\
        \bignorm{\frac{\partial z_W^j}{\partial W_i^{p,q}}\Bigg|_{W_k+U_k,\forall 1\leq k\leq j}}&\leq(1+\frac{1}{L})^{2j-1}\kappa_0^{j}\max_{x\in\cX}\bignorm{x}\frac{1}{\bignorms{W_i}}\prod_{h=1}^{j}\bignorms{W_h}. \label{eq_grad_rule_4}
    \end{align}

    Then, the Frobenius norm of the Jacobian is at most
    \begin{small}
    \begin{align}
        &\bignormFro{\frac{\partial z_W^{j+1}}{\partial W_i}\Bigg|_{W_k+U_k,\forall 1\leq k\leq j+1} - \frac{\partial z_W^{j+1}}{\partial W_i}} \nonumber \\
        \leq&~ d_{i}d_{i-1}\kappa_1\bignorms{W_{j+1}+U_{j+1}}\bignorm{z_{W+U}^j - z_W^j} + \bignorms{U_{j+1}}\bignorm{z_W^j} \bignorms{W_{j+1}+U_{j+1}}\bignorm{\frac{\partial z_W^j}{\partial W_i^{p,q}}\Bigg|_{W_k+U_k}}\label{eq_grad_5} \\
        &~+ d_{j+1}\kappa_0\bignorms{W_{j+1}+U_{j+1}}\bignormFro{\frac{\partial z_W^{j}}{\partial W_i^{p,q}}\Bigg|_{W_k+U_k,\forall 1\leq k\leq j} - \frac{\partial z_W^{j}}{\partial W_i^{p,q}}} + d_{i}d_{i-1}\bignorms{U_{j+1}}\bignorm{\frac{\partial z_W^{j}}{\partial W_i^{p,q}}} \label{eq_grad_6}
    \end{align}
    \end{small}%

    From equation \eqref{eq_grad_rule_1}, \eqref{eq_grad_rule_2}, \eqref{eq_grad_rule_4}, we have equation \eqref{eq_grad_5} will be at most
    \begin{align}\label{eq_grad_10}
        \Big(1 + \frac{1}{L}\Big)^{3j+3}\kappa_1\kappa_0^{2j+1}\Big(\max_{x\in\cX}\bignorm{x}\frac{1}{\bignorms{W_i}}\prod_{l=i-1}^{j+1}d_l\prod_{h=1}^{j+1}\bignorms{W_h}\Big)
        \Big(\max_{x\in\cX}\bignorm{x}\prod_{h=1}^{j+1}\bignorms{W_h}\sum_{h=1}^{j+1}\frac{\bignorms{U_h}}{\bignorms{W_h}}\Big).
    \end{align}
    From the induction and equation \eqref{eq_grad_rule_1} and \eqref{eq_grad_rule_3}, the term in equation \eqref{eq_grad_6} will be at most
    \begin{small}
    \begin{align}\label{eq_grad_11}
        &(j-i+2)\Big(1 + \frac{1}{L}\Big)^{3j+3}\kappa_1\kappa_0^{2j+1}\frac{\max_{x\in\cX}\bignorm{x}}{\bignorms{W_i}}\prod_{l=i-1}^{j+1}d_l\prod_{h=1}^{j+1}\bignorms{W_h} \max_{x\in\cX}\bignorm{x}\prod_{h=1}^{j+1}\bignorms{W_h}\sum_{h=1}^{j+1}\frac{\bignorms{U_h}}{\bignorms{W_h}}
    \end{align}
    \end{small}
    After we combine equation \eqref{eq_grad_10} and \eqref{eq_grad_11}, the proof is complete.
\end{proof}

\subsubsection{Lipschitz constant of the Hessian}

Based on the above two results, we state the proof of Proposition \ref{prop_hessian_lip}.

\begin{proof}[Proof of Proposition \ref{prop_hessian_lip}]
    We will prove using induction. Let $z_W^j$ be an $j$ layer neural network with parameters $W$.
    The change in the Hessian output of the $j$ layer network $z_W^j$ for $W_i$ under perturbation $U$ can be bounded as follows: 
    \begin{align}\label{eq_hess_induction}
        \bignormFro{\bH_i[z_W^j]\Big|_{W_k+U_k,\forall 1\leq k\leq j} - \bH_i[z_W^j]} 
        \leq G_{i,j}\Big(\max_{x\in\cX}\bignorm{x}\prod_{h=1}^{j}\bignorms{W_h}\sum_{h=1}^{j}\frac{\bignorms{U_h}}{\bignorms{W_h}}\Big).
    \end{align}
    For any $j = 1,2,\cdots, L$, 
    $G_{i,j}$ has the following equation:
    \begin{align*}
        G_{i,j} = \frac{3}{2}(j-i+2)^2\Big(1+\frac{1}{L}\Big)^{6j}\kappa_2\kappa_1^2\kappa_0^{3j-3}\prod_{l=i-1}^{j}d_l^3\Big(\max_{x\in\cX}\bignorm{x}\frac{1}{\bignorms{W_i}}\prod_{h=1}^{j}\bignorms{W_h}\Big)^2,
    \end{align*}
    For any fixed $i = 1,2,\cdots, L$, we will prove using induction for $j$. If $j < i$, we have $\frac{\partial^2}{\partial W_i^2}z_W^j = 0$ since $z_W^j$ remains unchanged regardless of any change of $W_i$. When $j = i$, we have
    \begin{align*}
        &\bignorm{\frac{\partial^2 z_W^i}{\partial W_i^{p,q} \partial W_i^{s,t}}\Bigg|_{W_k+U_k,\forall 1\leq k\leq i} - \frac{\partial^2 z_W^i}{\partial W_i^{p,q} \partial W_i^{s,t}}} \\
        =& \bigabs{\big([z^{i-1}_{W+U}]^q [z^{i-1}_{W+U}]^t\big)\phi_i''\Bigg(\sum_{r=1}^{d_{i-1}}\Big((W_i^{p,r} + U_i^{p,r}) [z_{W+U}^{i-1}]^r\Big)\Bigg) - ([z_W^{i-1}]^q[z_W^{i-1}]^t)\phi_i''\Big(\sum_{r=1}^{d_{i-1}}W_i^{p,r}[z_W^{i-1}]^r\Big)},
    \end{align*}
    where $p=s$. If $p\neq s$, we have $\frac{\partial^2}{\partial W_i^{p,q} \partial W_i^{s,t}}z_W^i = 0$, where $W_i^{p,q}, U_i^{p,q}$ are the $(p,q)$'th entry of $W_i$ and $U_i$.

    Let $[z_W^{i-1}]^q$ be the $q$'th entry of $z_W^{i-1}$. By the triangle inequality, the above equation is at most
    \begin{align}
        &\bigabs{[z_{W+U}^{i-1}]^q[z_{W+U}^{i-1}]^t} \bigabs{\phi_i''\Big(\sum_{r=1}^{d_{i-1}}\Big((W_i^{p,r} + U_i^{p,r}) [z_{W+U}^{i-1}]^r\Big) - \phi_i''\Big(\sum_{r=1}^{d_{i-1}}W_i^{p,r}[z_W^{i-1}]^r\Big)} \label{eq_hess_1}\\
        &+ \bigabs{[z_{W+U}^{i-1}]^q[z_{W+U}^{i-1}]^t - [z_W^{i-1}]^q[z_{W}^{i-1}]^t}\bigabs{\phi_i''\Big(\sum_{r=1}^{d_{i-1}}W_i^{p,r}[z_W^{i-1}]^r\Big)}. \label{eq_hess_2}
    \end{align}
    Since $\phi_i''$ is $\kappa_2$-Lipschitz by Assumption \ref{assumption_1}, equation \eqref{eq_hess_1} is at most
    \begin{align}\label{eq_hess_8}
        \bigabs{[z_{W+U}^{i-1}]^q[z_{W+U}^{i-1}]^t} \cdot\kappa_2\bigabs{\sum_{r=1}^{d_{i-1}}(W_i^{p,r} + U_i^{p,r}) [z_{W+U}^{i-1}]^r - \sum_{r=1}^{d_{i-1}}W_i^{p,r}[z_W^{i-1}]^r}.
    \end{align}
    Again by Assumption \ref{assumption_1}, $\bigabs{\phi_i''(\cdot)}\leq \kappa_1$ following the Lipschitz-continuity of the second-order derivatives.
    Thus, equation \eqref{eq_hess_2} is at most
    \begin{align}
        &\bigabs{[z_{W+U}^{i-1}]^q[z_{W+U}^{i-1}]^t - [z_W^{i-1}]^q[z_W^{i-1}]^t}\cdot\kappa_1 \nonumber \\
        \leq &\Bigg(\bigabs{[z_{W+U}^{i-1}]^q[z_{W+U}^{i-1}]^t - [z_{W}^{i-1}]^q[z_{W+U}^{i-1}]^t} + \bigabs{[z_{W}^{i-1}]^q[z_{W+U}^{i-1}]^t - [z_W^{i-1}]^q[z_W^{i-1}]^t}\Bigg)\cdot\kappa_1. \label{eq_hess_9}
    \end{align}
    We use the triangle inequality in the last step.
    
    Taking equation \eqref{eq_hess_8} and \eqref{eq_hess_9} together, the Frobenius norm of the Hessian is
    \begin{align*}
        &\bignormFro{\bH_i[z_W^i]\Big|_{W_k+U_k,\forall 1\leq k\leq i} - \bH_i[z_W^i]} \\
        \leq~& d_i\bignorm{z_{W+U}^{i-1}}^2\cdot k_2d_{i-1}\bignorm{\big(W_i+U_i\big)z_{W+U}^{i-1} - W_iz_W^{i-1}} + \kappa_1 d_i \bignorm{z_{W+U}^{i-1} - z_W^{i-1}}\Bigg(\bignorm{z_{W+U}^{i-1}} + \bignorm{z_{W}^{i-1}}\Bigg).
    \end{align*}
    By the Lipschitz continuity of $Z_{W+U}^{i-1}$ and equation \eqref{eq_result_lip} in Claim \ref{claim_lip}, we have
    \begin{align}
        \bignorm{z_{W+U}^{i-1}}\leq& \Big(1 + \frac{1}{L}\Big)^{i-1}\kappa_0^{i-1}\max_{x\in\cX}\bignorm{x}\prod_{h=1}^{i-1}\bignorms{W_h}. \label{eq_hess_rule_1} \\
        \bignorm{z_{W+U}^{i-1} - z_W^{i-1}}\leq& \Big(1 + \frac{1}{L}\Big)^{i-1} \kappa_0^{i-1}\max_{x\in\cX}\bignorm{x}\prod_{h=1}^{i-1}\bignorms{W_h}\sum_{h=1}^{i-1}\frac{\bignorms{U_h}}{\bignorms{W_h}} \label{eq_hess_rule_2}
    \end{align}
    By the triangle inequality, we have 
    \begin{align}\label{eq_hess_10}
        \bignorm{\big(W_i+U_i\big)z_{W+U}^{i-1} - W_iz_W^{i-1}}\leq \bignorms{W_i+U_i}\bignorm{z_{W+U}^{i-1} - z_W^{i-1}} + \bignorms{U_i}\bignorm{z_W^{i-1}},
    \end{align}
    From equations \eqref{eq_hess_rule_1}, \eqref{eq_hess_rule_2}, and \eqref{eq_hess_10}, we finally obtain that the Hessian of $z_W^i$ with respect to $W_i$ is at most
    \begin{align*}
        &\bignormFro{\bH_i[z_W^i]\Big|_{W_k+U_k,\forall 1\leq k\leq i} - \bH_i[z_W^i]} \\
        \leq~& 3\Big(1 + \frac{1}{L}\Big)^{3i}\kappa_2\kappa_1\kappa_0^{3i-3}d_id_{i-1}\Bigg(\max_{x\in\cX}\bignorm{x}\prod_{h=1}^{i-1}\bignorms{W_h}\Bigg)^2 \Bigg(\max_{x\in\cX}\bignorm{x}\prod_{h=1}^{i}\bignorms{W_h}\sum_{h=1}^i\frac{\bignorms{U_h}}{\bignorms{W_h}}\Bigg).
    \end{align*}
    Thus, we have known that equation \eqref{eq_hess_induction} is correct when $j = i$ for $\bH_i$.

    Assume that equation \eqref{eq_hess_induction} is correct for any $i$ up to $j \ge i$ for $\bH_i$. Thus, the following inequality holds:
    \begin{align*}
        & \bignormFro{\bH_i[z_W^j]\Big|_{W_k+U_k,\forall k} - \bH_i[z_W^j]} \\
        \leq& C \prod_{l=i-1}^{j}d_l^3\Bigg(\max_{x\in\cX}\bignorm{x}\frac{1}{\bignorms{W_i}}\prod_{h=1}^{j}\bignorms{W_h}\Bigg)^2 \Bigg(\max_{x\in\cX}\bignorm{x}\prod_{h=1}^{j}\bignorms{W_h}\sum_{h=1}^{j}\frac{\bignorms{U_h}}{\bignorms{W_h}}\Bigg),
    \end{align*}
    where $C = \frac{3}{2}(j-i+2)^2\Big(1+\frac{1}{L}\Big)^{6j}\kappa_2\kappa_1^2\kappa_0^{3(j-1)}$.

    Then, we will consider the second-order derivative of $z_W^{j+1} = \phi_{j+1}(W_{j+1}z_W^{j})$ with respect to $W_i^{p,q}$ and $W_i^{s,t}$:
    {\begin{align*}
        &\frac{\partial^2 \phi_{j+1}(W_{j+1}z_W^{j})}{\partial W_i^{p,q}\partial W_i^{s,t}} \\
        =&
        \phi_{j+1}''(W_{j+1}z_W^{j}) \odot \Bigg(W_{j+1}\frac{\partial z_W^{j}}{\partial W_i^{p,q}}\Bigg) \odot \Bigg(W_{j+1}\frac{\partial z_W^{j}}{\partial W_i^{s,t}}\Bigg)  
        + \phi_{j+1}'(W_{j+1}z_W^{j})\odot \Bigg(W_{j+1}\frac{\partial^2 z_W^{j}}{\partial W_i^{p,q}\partial W_i^{s,t}}\Bigg),
    \end{align*}}%
    where $\odot$ is the Hadamard product operator between two vectors. Then, we have
    \begin{small}
    \begin{align*}
        &\bignorm{\frac{\partial^2 z_W^{j+1}}{\partial W_i^{p,q} \partial W_i^{s,t}}\Bigg|_{W_k+U_k,\forall 1\leq k\leq j+1} - \frac{\partial^2 z_W^{j+1}}{\partial W_i^{p,q} \partial W_i^{s,t}}} \\
        =~& \Bigg\|\phi_{j+1}''((W_{j+1}+U_{j+1})z_{W+U}^{j}) \odot \Bigg((W_{j+1}+U_{j+1})\frac{\partial z_W^{j}}{\partial W_i^{p,q}}\Bigg|_{W_k+U_k}\Bigg) \odot (W_{j+1}+U_{j+1})\frac{\partial z_W^{j}}{\partial W_i^{s,t}}\Bigg|_{W_k+U_k,\forall 1\leq k\leq j} \\
        &~+ \phi_{j+1}'((W_{j+1}+U_{j+1})z_{W+U}^{j})\odot (W_{j+1}+U_{j+1})\frac{\partial^2 z_W^{j}}{\partial W_i^{p,q}\partial W_i^{s,t}}\Bigg|_{W_k+U_k,\forall 1\leq k\leq j} \\
        &~- \phi_{j+1}''(W_{j+1}z_W^{j}) \odot \Bigg(W_{j+1}\frac{\partial z_W^{j}}{\partial W_i^{p,q}}\Bigg) \odot \Bigg(W_{j+1}\frac{\partial z_W^{j}}{\partial W_i^{s,t}}\Bigg) - \phi_{j+1}'(W_{j+1}z_W^{j})\odot \Bigg(W_{j+1}\frac{\partial^2 z_W^{j}}{\partial W_i^{p,q}\partial W_i^{s,t}}\Bigg)\Bigg\|
    \end{align*}
    \end{small}%

    By the Cauchy-Schwarz inequality and the triangle inequality, the above equation is, at most
    \begin{small}
    \begin{align}
        &\bignorm{\phi_{j+1}''((W_{j+1}+U_{j+1})z_{W+U}^{j}) - \phi_{j+1}''(W_{j+1}z_{W}^{j})}\cdot \bignorm{(W_{j+1}+U_{j+1})\frac{\partial z_W^{j}}{\partial W_i^{p,q}}}\cdot\bignorm{(W_{j+1}+U_{j+1})\frac{\partial z_W^{j}}{\partial W_i^{s,t}}} \label{eq_hess_3} \\
        +& \bignorm{\phi_{j+1}''(W_{j+1}z_{W}^{j})}\cdot\bignorm{(W_{j+1}+U_{j+1})\frac{\partial z_W^{j}}{\partial W_i^{p,q}}\Bigg|_{W_k+U_k} - W_{j+1}\frac{\partial z_W^{j}}{\partial W_i^{p,q}}}\cdot\bignorm{(W_{j+1}+U_{j+1})\frac{\partial z_W^{j}}{\partial W_i^{s,t}}\Bigg|_{W_k+U_k}} \label{eq_hess_4} \\
        +& \bignorm{\phi_{j+1}''(W_{j+1}z_{W}^{j})}\cdot\bignorm{W_{j+1}\frac{\partial z_W^{j}}{\partial W_i^{p,q}}}\cdot\bignorm{(W_{j+1}+U_{j+1})\frac{\partial z_W^{j}}{\partial W_i^{s,t}}\Bigg|_{W_k+U_k,\forall 1\leq k\leq j} - W_{j+1}\frac{\partial z_W^{j}}{\partial W_i^{s,t}}} \label{eq_hess_5} \\
        +& \bignorm{\phi_{j+1}'((W_{j+1}+U_{j+1})z_{W+U}^{j}) - \phi_{j+1}'(W_{j+1}z_{W}^{j})}\cdot\bignorm{(W_{j+1}+U_{j+1})\frac{\partial^2 z_W^{j}}{\partial W_i^{p,q}\partial W_i^{s,t}}\Bigg|_{W_k+U_k,\forall 1\leq k\leq j}} \label{eq_hess_6} \\
        +& \bignorm{\phi_{j+1}'(W_{j+1}z_W^{j})}\cdot \bignorm{(W_{j+1}+U_{j+1})\frac{\partial^2 z_W^{j}}{\partial W_i^{p,q}\partial W_i^{s,t}}\Bigg|_{W_k+U_k,\forall 1\leq k\leq j} - W_{j+1}\frac{\partial^2 z_W^{j}}{\partial W_i^{p,q}\partial W_i^{s,t}}}. \label{eq_hess_7}
    \end{align}
    \end{small}%

    Next, we will deal with the above quantities one by one.
    The idea is essentially using the Lipschitz-continuity properties of the activation functions.

    Using the triangle inequality and the fact that $\phi''(\cdot)$ is $\kappa_2$-Lipschitz, equation \eqref{eq_hess_3} is at most $\kappa_2$ times
    \begin{small}
    \begin{align*}
        \Bigg(\bignorms{W_{j+1}+U_{j+1}}  \bignorm{z_{W+U}^{j} - z_{W}^{j}} + \bignorms{U_{j+1}}  \bignorm{z_{W}^{j}}\Bigg)\bignorms{W_{j+1}+U_{j+1}}^2  \bignorm{\frac{\partial z_W^{j}}{\partial W_i^{p,q}}\Bigg|_{W_k+U_k}} \bignorm{\frac{\partial}{\partial W_i^{s,t}}z_W^{j}\Bigg|_{W_k+U_k}}
    \end{align*}
    \end{small}%
    Using the triangle inequality and the fact that $|\phi''(\cdot)|\leq\kappa_1$, equation \eqref{eq_hess_4} is at most $\kappa_1d_{j+1}$ times
    \begin{small}
    \begin{align*}
        \Bigg(\bignorms{W_{j+1}+U_{j+1}}  \bignorm{\frac{\partial z_W^{j}}{\partial W_i^{p,q}}\Bigg|_{W_k+U_k} - \frac{\partial z_W^{j}}{\partial W_i^{p,q}}} + \bignorms{U_{j+1}} \bignorm{\frac{\partial z_W^{j}}{\partial W_i^{p,q}}}\Bigg)\bignorms{W_{j+1}+U_{j+1}}  \bignorm{\frac{\partial z_W^{j}}{\partial W_i^{s,t}}\Bigg|_{W_k+U_k}}
    \end{align*}
    \end{small}%
    Similarly, using the triangle inequality and $|\phi''(\cdot)|\leq\kappa_1$, equation \eqref{eq_hess_5} is at most $\kappa_1d_{j+1}$ times
    \begin{align*}
        \bignorms{W_{j+1}}\bignorm{\frac{\partial z_W^{j}}{\partial W_i^{p,q}}}\Bigg(\bignorms{W_{j+1}+U_{j+1}} \cdot \bignorm{\frac{\partial z_W^{j}}{\partial W_i^{s,t}}\Bigg|_{W_k+U_k,\forall k} - \frac{\partial z_W^{j}}{\partial W_i^{s,t}}} + \bignorms{U_{j+1}} \cdot \bignorm{\frac{\partial z_W^{j}}{\partial W_i^{s,t}}}\Bigg)
    \end{align*}
    Lastly, using the fact that $\phi'(\cdot)$ is $\kappa_1$-Lipschitz, equation \eqref{eq_hess_6} is at most
    \begin{align*}
        \kappa_1\Bigg(\bignorms{W_{j+1}+U_{j+1}} \bignorm{z_{W+U}^{j} - z_{W}^{j}} + \bignorms{U_{j+1}}  \bignorm{z_{W}^{j}}\Bigg)\bignorms{W_{j+1}+U_{j+1}}  \bignorm{\frac{\partial^2 z_W^{j}}{\partial W_i^{p,q}\partial W_i^{s,t}}\Bigg|_{W_k+U_k}}
    \end{align*}
    Using the fact that $|\phi'(\cdot)|\leq\kappa_0$, equation \eqref{eq_hess_7} is at most
    \begin{align*}
        \kappa_0 d_{j+1}\Bigg(\bignorms{W_{j+1}+U_{j+1}} \cdot \bignorm{\frac{\partial^2 z_W^{j}}{\partial W_i^{p,q}\partial W_i^{s,t}}\Bigg|_{W_k+U_k} - \frac{\partial^2 z_W^{j}}{\partial W_i^{p,q}\partial W_i^{s,t}}} + \bignorms{U_{j+1}}\cdot \bignorm{\frac{\partial^2 z_W^{j}}{\partial W_i^{p,q}\partial W_i^{s,t}}}\Bigg)
    \end{align*}
    From Claim \ref{claim_lip}, we know that the norm of the derivative of $z_W^j$ with respect to $W_i^{p,q}$ is bounded by its Lipschitz-continuity in equation \eqref{eq_result_lip}:
    \begin{align}
        \bignorm{\frac{\partial}{\partial W_i^{p,q}}z_W^{j}}&\leq (1+\frac{1}{L})^{j}\kappa_0^{j}\max_{x\in\cX}\bignorm{x}\frac{1}{\bignorms{W_i}}\prod_{h=1}^{j}\bignorms{W_h}. \label{eq_hess_rule_3} \\
        \bignorm{\frac{\partial}{\partial W_i^{p,q}}z_W^j\Big|_{W_k+U_k,\forall 1\leq k\leq j}}&\leq(1+\frac{1}{L})^{2j-1}\kappa_0^{j}\max_{x\in\cX}\bignorm{x}\frac{1}{\bignorms{W_i}}\prod_{h=1}^{j}\bignorms{W_h}. \label{eq_hess_rule_4}
    \end{align}
    From Claim \ref{prop_grad}, the norm of the twice-derivative of $z_W^j$ with respect to $W_i^{p,q}$ and $W_i^{s,t}$ is bounded by its Lipschitz-continuity in equation \eqref{eq_result_grad}:
    \begin{align}
        \bignorm{\frac{\partial^2}{\partial W_i^{p,q}\partial W_i^{s,t}}z_W^{j}} &\leq A_{i,j}\max_{x\in\cX}\bignorm{x}\frac{1}{\bignorms{W_i}}\prod_{h=1}^{j}\bignorms{W_h}. \label{eq_hess_rule_5}\\
        \bignorm{\frac{\partial^2}{\partial W_i^{p,q}\partial W_i^{s,t}}z_W^{j}\Bigg|_{W_k+U_k,\forall 1\leq k\leq j}} &\leq \Big(1+\frac{1}{L}\Big)^{2j-2}A_{i,j}\max_{x\in\cX}\bignorm{x}\frac{1}{\bignorms{W_i}}\prod_{h=1}^{j}\bignorms{W_h}. \label{eq_hess_rule_6}
    \end{align}
    Then, the Frobenius norm of the Hessian $\bH_i$ is at most
    \begin{align*}
        &\bignormFro{\bH_i[z_W^{j+1}]\Big|_{W_k+U_k,\forall 1\leq k\leq j+1} - \bH_i[z_W^{j+1}]} \\
        \leq& (d_id_{i-1})^2 \bignorm{\frac{\partial^2}{\partial W_i^{p,q} \partial W_i^{s,t}}z_W^{j+1}\Bigg|_{W_k+U_k} - \frac{\partial^2}{\partial W_i^{p,q} \partial W_i^{s,t}}z_W^{j+1}}
    \end{align*}
    Since we have listed the upper bound of all the terms in the Hessian for $W_i$ from equations \eqref{eq_hess_rule_1}, \eqref{eq_hess_rule_2}, \eqref{eq_hess_rule_3}, \eqref{eq_hess_rule_4}, \eqref{eq_hess_rule_5}, \eqref{eq_hess_rule_6}, and the induction, equation \eqref{eq_hess_3} times $(d_id_{i-1})^2$ is at most
    \begin{small}
    \begin{align*}
        \Big(1+\frac{1}{L}\Big)^{6(j+1)}\kappa_2\kappa_1^2\kappa_0^{3j}\prod_{l=i-1}^{j+1}d_l^3\Bigg(\max_{x\in\cX}\bignorm{x}\frac{1}{\bignorms{W_i}}\prod_{h=1}^{j+1}\bignorms{W_h}\Bigg)^2 \Bigg(\max_{x\in\cX}\bignorm{x}\prod_{h=1}^{j+1}\bignorms{W_h}\sum_{h=1}^{j+1}\frac{\bignorms{U_h}}{\bignorms{W_h}}\Bigg).
    \end{align*}
    \end{small}%
    Equation \eqref{eq_hess_4},\eqref{eq_hess_5}, and \eqref{eq_hess_6} times $(d_id_{i-1})^2$ is at most
    {\small\begin{align*}
        (j-i+2)\Big(1+\frac{1}{L}\Big)^{6(j+1)}\kappa_2\kappa_1^2\kappa_0^{3j}\prod_{l=i-1}^{j+1}d_l^3\Bigg(\frac{\max_{x\in\cX}\bignorms{x}}{\bignorms{W_i}}\prod_{h=1}^{j+1}\bignorms{W_h}\Bigg)^2 \max_{x\in\cX}\bignorm{x}\prod_{h=1}^{j+1}\bignorms{W_h}\sum_{h=1}^{j+1}\frac{\bignorms{U_h}}{\bignorms{W_h}}.
    \end{align*}}%
    Equation \eqref{eq_hess_7} times $(d_id_{i-1})^2$ is at most
    {\small\begin{align*}
        \frac{3}{2}(j-i+2)^2\Big(1+\frac{1}{L}\Big)^{6(j+1)}\kappa_2\kappa_1^2\kappa_0^{3j}\prod_{l=i-1}^{j+1}d_l^3\Bigg(\frac{\max_{x\in\cX}\bignorm{x}}{\bignorms{W_i}}\prod_{h=1}^{j+1}\bignorms{W_h}\Bigg)^2 \max_{x\in\cX}\bignorm{x}\prod_{h=1}^{j+1}\bignorms{W_h}\sum_{h=1}^{j+1}\frac{\bignorms{U_h}}{\bignorms{W_h}}.
    \end{align*}}%
    By combining the above five cases, we have finished the proof of the induction step.
\end{proof}

\paragraph{Remark.} We remark that the bulk of the work in this proof is to get an explicit formula of $G$ --- the Lipschitz bound of the neural network Hessian. If one wants to get the order rather than its exact value, the calculation can be significantly simplified.

\subsection{Proof of Lemma \ref{lemma_perturbation}}\label{proof_perturbation}

 \begin{claim}\label{lemma_taylor}
    Let $f_W$ be a feed-forward neural network with the activation function $\phi_i$ for any $i=1,2,\cdots, L$, parameterized by the layer weight matrices $W$.
    Let $U$ be any matrices with the same dimension as $W$.
    For any $W$ and $U$, the following identity holds
    \begin{align*}%
        \ell(f_{W + U}(x), y) =& \ell(f_W(x), y) + \langle U, \nabla_W (\ell(f_W(x), y))\rangle \\
        &+ \langle \vect{U}, \bH[\ell(f_W(x),y)] \vect{U}\rangle + R_2(\ell(f_W(x),y)).
    \end{align*}
    where $R_2(\ell(f_W(x),y)))$ is a second-order error term in the Taylor's expansion.
\end{claim}

\begin{proof}
    The proof of Claim \ref{lemma_taylor} follows by the fact that the activation $\phi_i(x)$ is twice-differentiable.
    From the mean value theorem, let $\eta$ be a matrix that has the same dimension as $W$ and $U$. There must exist an $\eta$ between $W$ and $U+W$ so that the following equality holds:
    \begin{align*}
        R_2(\ell(f_W(x),y)) = \langle \vect{U}, \big(\bH[\ell(f_W(x),y)]\mid_{\eta} - \bH[\ell(f_W(x),y)]\big)\vect{U}\rangle.
    \end{align*}
    This completes the proof of the claim.
\end{proof}

Based on the above claim, we prove Taylor's expansion of noise stability.
 
\begin{proof}[Proof of Lemma \ref{lemma_perturbation}]
We take the expectation over $U$ for both sides of the equation in Lemma \ref{lemma_taylor}. The result becomes
\begin{align*}
    \mathbb{E}_U[\ell(f_{W+U}(x),y)] = & ~ \mathbb{E}_U\big[\ell(f_W(x),y) + \langle U, \nabla_W (\ell(f_W(x),y))\rangle \\
    &+ \langle \vect{U},\bH[\ell(f_W(x),y)]\vect{U}\rangle + R_2(\ell(f_W(x),y))\big].
\end{align*}
Then, we use the notation of the distribution $\cQ$ on $\mathbb{E}_U[\ell(f_{W+U}(x),y)]$. We have
\begin{align*}
    \ell_\cQ(f_W(x),y) =& \mathbb{E}_U[\ell(f_W(x),y)] + \mathbb{E}_U[\langle U, \nabla_W(\ell(f_W(x),y))\rangle] \\
    &+ \mathbb{E}_U[\langle \vect{U},\mathbf{H}[\ell(f_W(x),y)]\vect{U}\rangle] + \mathbb{E}_U[ R_2(\ell(f_W(x),y))].
\end{align*}
Since $\mathbb{E}[U_i] = 0$ for every $i = 1,2,\cdots,L$, the first-order term will be
\begin{align}
    \mathbb{E}_U[\langle U, \nabla_W(\ell(f_W(x),y))\rangle] = \sum_{i=1}^L\mathbb{E}_{U_i}[\langle U_i, \nabla_W(\ell(f_W(x),y))\rangle]  = 0 \label{eq_taylor_1}.
\end{align} 
Since we assume that $U_i$ and $U_j$ are independent variables for any $i\neq j$, we have
\begin{align*}
    \mathbb{E}_U[\langle \vect{U},\mathbf{H}[\ell(f_W(x),y)]\vect{U}\rangle] 
    =& \sum_{i,j=1}^L \mathbb{E}_U[\langle \vect{U_i}, \frac{\partial^2}{\partial \vect{W_i}\partial \vect{W_j}}\ell(f_W(x),y)\vect{U_j}\rangle] \\
    =& \sum_{i=1}^L\mathbb{E}_U[\langle \vect{U_i},\mathbf{H}_i[\ell( f_W(x),y)]\vect{U_i}\rangle].
\end{align*}
For each $U_i$, let $\Sigma_i$ be the variance of $U_i$. The second-order term will be
\begin{align}
    \mathbb{E}_{U_i}[\langle \vect{U_i},\bH_i[\ell( f_W(x),y)]\vect{U_i}\rangle] = \big\langle \Sigma_i,\mathbf{H}_i[\ell(f_W(x),y)]\big\rangle. \label{eq_taylor_2}
\end{align}
The expectation of the error term $R_2(\ell(f_W(x),y))$ be the following
\begin{align*}
    \mathbb{E}_U[R_2(\ell(f_W(x),y))] &= \mathbb{E}_U[\langle \vect{U}\vect{U}^\top, \big(\bH[\ell(f_W(x),y)]\mid_\eta - \bH[\ell(f_W(x),y)]\big)\rangle] \\
    &= \sum_{i=1}^L\mathbb{E}_{U_i}[\langle \vect{U_i}\vect{U_i}^\top, \big(\bH_i[\ell(f_W(x),y)]\mid_{\eta_i} - \bH_i[\ell(f_W(x),y)]\big)\rangle].
\end{align*}
From Proposition \ref{prop_hessian_lip}, let $\frac{\sqrt{3}}{9}C_1$ equal to the following number
\begin{align*}
    \frac{3}{2}(L+1)^2e^6\kappa_2\kappa_1^2\kappa_0^{3L}\Big(\max_{x\in\cX}\bignorms{x}\max_{1\leq i\leq L}\frac{1}{\bignorms{W_i}}\prod_{h=0}^{L}d_h\prod_{h=1}^{L}\bignorms{W_h}\Big)^3.
\end{align*}
Hence, we have the following inequality:
\begin{align*}
    &\mathbb{E}_{U_i}\Big[\langle \vect{U_i}\vect{U_i}^\top, \big(\bH_i[\ell(f_W(x),y)]\mid_{\eta_i} - \bH_i[\ell(f_W(x),y)]\big)\rangle\Big] \\
    \leq\,& \mathbb{E}_{U_i}\Big[\bignormFro{U_i}^2\bignormFro{\mathbf{H}_i[\ell(f_W(x),y)]\mid_{\eta_i} - \bH_i[\ell(f_W(x),y)]}\Big] \\
    \leq\,& \mathbb{E}_{U_i}\Big[\bignormFro{U_i}^2 \frac{\sqrt{3}}{9}C_1\bignormFro{U_i}\Big] \\
    =\,& \frac{\sqrt{3}}{9}C_1 \mathbb{E}_{U_i}\Big[\bignormFro{U_i}^3\Big].
\end{align*}
By standard facts, we have that $(\mathbb{E}_{U_i}\bignormFro{U_i}^3)^{1/3}\leq \sqrt{3}\bignormFro{\Sigma_i}$.
As a result, the error term $R_2(\ell(f_W(x),y))$ is smaller than 
$    \frac{\sqrt{3}}{9}C_1\sum_{i=1}^L(\sqrt{3}\bignormFro{\Sigma_i})^3 = C_1\sum_{i=1}^L\bignormFro{\Sigma_i}^{3/2}.$
Thus, the proof is complete.
\end{proof}

\section{Experiment Details}\label{sec_add_exp}

The code repository for reproducing our experiments can be found at \url{https://github.com/VirtuosoResearch/Robust-fine-tuning}.
Below, we outline several aspects to complement the experiment section.

\subsection{Experimental setup for measuring generalization}\label{sec_add_setup}


Figure \ref{fig_intro} compares our Hessian distance measure with previous distance measures and generalization errors of fine-tuning. In this illustration, we fine-tune the ViT-Base model on the Clipart dataset with weakly-supervised label noise. 
We control the model distance from the initialization $v_i$ in fine-tuning and sweep the range of distance constraints in $\{0.1, 0.2, 0.3, 0.4, 0.5, 0.6, 0.7, 0.8, 0.9, 1.0\}$. 
For the distance measure, we evaluate the generalization bound from \citet{pitas2017pac}. For the scale of the bound is orders of magnitude larger than the generalization error, we show their values divided by $10^{74}$ in Figure \ref{fig_intro}. 
For Hessian distance measure, we evaluate the quantity $${ \sum_{i=1}^L \sqrt{C \cdot\cH_i}/{\sqrt n}}$$ from Equation \eqref{eq_main}.
The difference between averaged test and training losses is the generalization error.

Figure \ref{fig_reg_res1} and \ref{fig_reg_res2} compare our generalization measure from Theorem \ref{theorem_hessian} with the empirical generalization errors in fine-tuning ResNet.
We use ResNet-50 models fine-tuned on the Indoor and CalTech-256 data sets.
We study the following regularization method: (1) fine-tuning with early stopping; (2) weight decay; (3) label smoothing; (4) Mixup; (5) $\ell^2$ soft penalty; and (6) $\ell^2$ projected gradient methods. 
Figure \ref{fig_reg_bert} compares our generalization measure with the generalization errors in BERT fine-tuning. We consider BERT-Base on the MRPC dataset. We study the following regularization methods:
(1) early stopping;
(2) weight decay.
(3) $\ell^2$ soft penalty.
(4) $\ell^2$ projected gradient descent \cite{gouk2020distance}.
Figure \ref{fig_avg_d} numerically measures our Hessian distances $\sum_{i}^L \sqrt{\cH_{i}}$ on ResNet-50 and BERT-Base models. 
We measure these quantities for models fine-tuned with early stopping (implicit regularization) and $\ell^2$ projection gradient methods (explicit regularization).
Hyper-parameters are selected based on the accuracy of the validation dataset.
Figure \ref{fig:heatmap} measures the traces of the loss's Hessian by varying the noisy label.
We fine-tune both ResNet-18 and ResNet-50 on CIFAR-10 and CIFAR-100.
Each subfigure is a matrix corresponding to ten classes.
We use the ten classes for CIFAR-10 and randomly select ten classes for CIFAR-100.
The $(i, j)$-th entry of the matrix is the trace of the loss's Hessian calculated from using noisy label $\tilde{y} = j$ on samples with true label $i$.
For presentation, we select samples whose loss value under the true label is less than 1e-4. The trace values are then averaged over 200 samples from each class.
 
Figure \ref{fig:reduced_hessians} shows our Hessian distance measure of models fine-tuned with different algorithms, including early stopping, label smoothing, Mixup, SupCon, SAM, and our algorithm. 
We fine-tune ResNet-18 models on the Clipart dataset with weakly-supervised label noise. The Hessian distance measures are evaluated as $\sum_{i=1}^L\sqrt{\cH_i}$ from Equation \ref{eq_main}.

Table \ref{tab:bound_measurement} compares our generalization measure from Equation \ref{eq_main} with previous generalization bounds. We report the results of fine-tuning ResNet-50 on CIFAR-10 and CIFAR-100. We also include the results of fine-tuning BERT-Base on MRPC and SST-2. We describe the precise quantity for previous results:

\begin{itemize}
    \item Theorem 1 from \citet{gouk2020distance}: $ \frac{1}{\sqrt{n}}\prod_{i=1}^{L}(2 ||W_i||_{\infty}) \sum_{i=1}^{L}(\frac{||W_i - W_i^{(s)}||_{\infty}}{||W_i||_{\infty}})$ where $||\cdot||_{\infty}$ is the maximum absolute row sum matrix norm.
    \item Theorem 4.1 from~\citet{li2021improved}: $\sqrt{
    \frac{1}{\epsilon^2 n} \left ( \sum_{i=1}^L \frac{\prod_{j=1}^L (B_j + D_j) }{B_i + D_i} \right )^2 \left ( \sum_{i=1}^L D_i^2 \right) 
    }$, where $\epsilon$ is a small constant value set as $0.1$ of the average training loss, $B_i = ||W_i^{(s)}||_2$, and $D_i = ||W_i - W_i^{(s)}||_F$.  
    \item Theorem 2.1 from \citet{long2020generalization}: $\sqrt{\frac{M}{n} \prod_{i=1}^{L}||W_i||_2^2  \sum_{i=1}^{L} ||W_i - W_i^{(s)} ||_2}$ where $M$ is the total number of trainable parameters.
    \item Theorem 1 from \citet{neyshabur2017pac}: $\sqrt{
    \frac{1}{\gamma^2 n} \prod_{i=1}^{L}||W_i||_2^2\sum_{i=1}^L\frac{||W_i||_F^2}{||W_i||_2^2}
    }$.
    \item Theorem 4.2 from~\citet{pitas2017pac}: 
    $\sqrt{
    \frac{1}{\gamma^2 n} \prod_{i=1}^{L}||W_i||_2^2\sum_{i=1}^L\frac{||W_i - W_i^{(s)}||_F^2}{||W_i||_2^2}
    }$.
    \item Theorem 5.1 from~\citet{arora2018stronger}: 
    $\sqrt{
    \frac{1}{\gamma^2 n} \max_x ||f_w(x)||_2^2 \sum_{i=1}^{L} \frac{\beta^2 c_i^2 \left \lceil \kappa/s \right \rceil }{\mu_{i\rightarrow}^2 \mu_i^2}
    }$.
\end{itemize}
The margin bounds require specifying the desired margin $\gamma$.
In Figure \ref{fig:margin}, we show the classification error rates, i.e., the margin loss values, of fine-tuned ResNet-50 on CIFAR-10 and CIFAR-100 about the margin. 
The results show that the classification error depends heavily on $\gamma$.
For a fair comparison, we select the largest margin so that the difference between margin loss and zero-one loss is less than 1\%.
This leads to a margin of 0.4 on CIFAR-10 and 0.01 on CIFAR-100.

\begin{figure*}[!t]
    \centering
    \begin{subfigure}[b]{0.475\textwidth}
        \centering
        \includegraphics[width=0.8\textwidth]{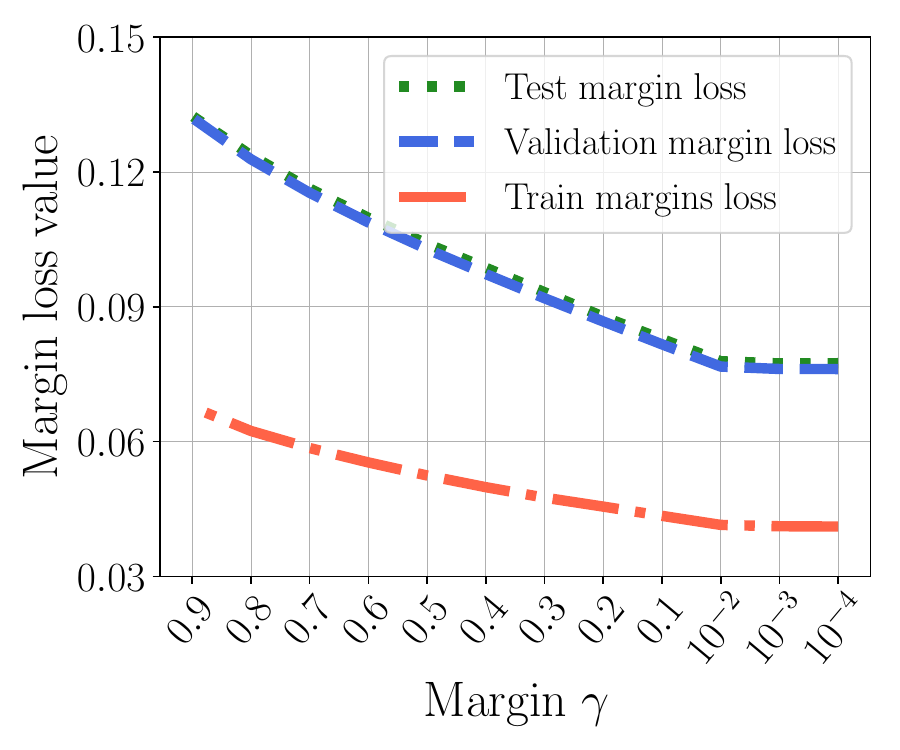}
    \end{subfigure}
    \begin{subfigure}[b]{0.475\textwidth}
         \centering
         \includegraphics[width=0.8\textwidth]{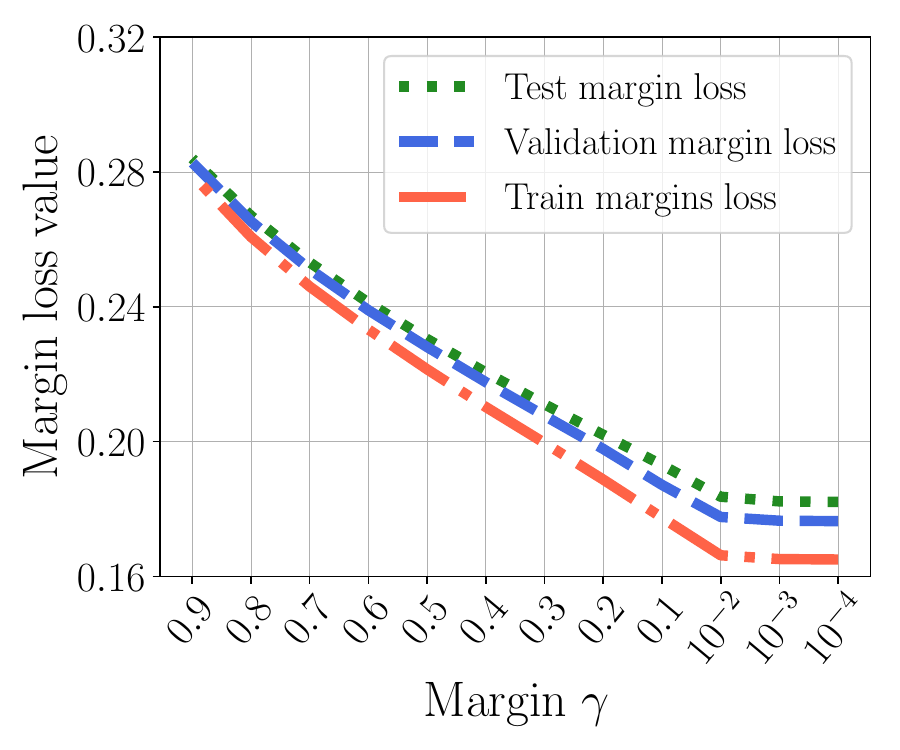}
    \end{subfigure}
    \caption{Margin loss with various margins for a fine-tuned network on CIFAR data sets.}\label{fig:margin}     
 \end{figure*}


In the experiments on image classification data sets, we fine-tune the model for 30 epochs. We use Adam optimizer with learning rate 1e-4 and decay the learning rate by ten every ten epochs.
In the experiments on text classification data sets, we fine-tune the BERT-Base model for 5 epochs. We use Adam optimizer with an initial learning rate of 5e-4 and then linearly decay the learning rate.

\subsection{Experimental setup for fine-tuning with noisy labels}\label{sec_add_exp_noisy}

\noindent\textbf{Datasets.} We evaluate our algorithm for both image and text classification tasks.
We use six image classification tasks derived from the DomainNet dataset. The DomainNet dataset contains images from 345 classes in 6 domains, including Clipart, Infograph, Painting, Quickdraw, Real, and Sketch. We use the same data processing as in \citet{mazzetto2021adversarial} and use five classes randomly selected from the 25 classes with the most instances.
We also use MRPC dataset from the GLUE benchmark.

We consider independent  and correlated label noise in our experiments. Independent label noise is generated by flipping the labels of a given proportion of training samples to other classes uniformly.
Correlated label noise is generated following the  \citet{mazzetto2021adversarial} for annotating images from the DomainNet dataset. For completeness, we include a description of the protocol.
For each domain, we learn a multi-class predictor for the five classes. The predictor is trained by fine-tuning a pretrained ResNet-18 network, using 60\% of the labeled data for that domain. For each domain, we consider the classifiers trained in other domains as weak supervision sources. We generate noisy labels by taking majority votes from weak supervision sources.  Table \ref{tab:dataset_statistics} shows the statistics of the dataset.

\begin{table}[h!]
\centering
\caption{Basic statistics for six data sets with noisy labels (cf. \citet{mazzetto2021adversarial}).}\label{tab:dataset_statistics}
\begin{scriptsize}
\begin{tabular}{@{}lcccccc@{}}
\toprule
  & \multicolumn{6}{c}{DomainNet} \\ \cmidrule(l){2-7} 
  & Clipart  & Infograph  & Painting  &  Quickdraw & Real  & Sketch \\ \midrule
Number of training Samples & 750 &	858 & 872 & 1500 & 1943 &	732 \\
Number of validation Samples & 250 & 286	& 291	& 500 & 648 & 245 \\
Number of test Samples & 250 & 287 & 291 & 500 & 648 & 245 \\
Noise rate & 0.4147 & 0.6329 & 0.4450 & 0.6054 & 0.3464 & 0.4768\\ \bottomrule
\end{tabular}
\end{scriptsize}
\end{table} 

\paragraph{Models.} For the experiments on the six tasks from DomainNet, we fine-tune ResNet-18 and ResNet-101 with the Adam optimizer.
We use an initial learning rate of $0.0001$ for $30$ epochs and decay the learning rate by $0.1$ every $10$ epochs.
For the experiments on MRPC, we fine-tune the RoBERTa-Base model for ten epochs.
We use the Adam optimizer with an initial learning rate of 5e-4 and then linearly decay the learning rate.
We report the Top-1 accuracy on the test set.
We average the result over ten random seeds.

\paragraph{Implementation.}
In our algorithm, we use the confusion matrix $F$ obtained by the estimation method from \citet{yao2020dual}. For applying regularization constraints, we use the layer-wise regularization method in \citet{li2021improved}.
The distance constraint parameter $D_i = \gamma^{i-1}D$ is set for layer $i$ in equation \eqref{eq_constraint}. $\gamma$ is a hyper-parameter controlling the scaling of the constraints, and $D$ is the constraint set for layer $1$. 
We search the distance constraint parameter $D$ in $[0.05, 10]$ and the scaling parameter $\gamma$ in $[1, 5]$.
For the results reported in Table \ref{tab:syn_and_ws_noise}, we searched the hyper-parameters for $30$ times via cross-validation. 

\paragraph{Hyper-parameters.} For the baselines, we report the results from running their open-sourced implementations.
For label smoothing, we search  $\alpha$ in $\{0.2, 0.4, 0.6, 0.8\}$.
For Mixup, we search  $\alpha$ in $\{0.2, 0.4, 0.6, 0.8, 1.0, 2.0\}$.
For FixMatch, we adopt it on fine-tuning from noisy labels. We set a starting epoch to apply the FixMatch algorithm. We search its pseudo-labeling threshold in $\{ 0.7, 0.8, 0.9 \}$. The starting epoch is searched in $\{2, 4, 6, 8\}$. 
For APL, we choose the active loss as normalized cross-entropy loss and the passive loss as reversed cross-entropy loss. We search the loss factor $\alpha$ in $\{1, 10, 100\}$ and $\beta$ in $\{0.1, 1.0, 10\}$.
For ELR, we search the momentum factor $\beta$ in $\{0.5, 0.7, 0.9, 0.99\}$ and the weight factor $\lambda$ in $\{0.05, 0.3, 0.5, 0.7\}$.
For SAT, the start epoch is searched in $\{3, 5, 8, 10, 13\}$, and the momentum is searched in $\{0.6, 0.8, 0.9, 0.99\}$.
For GJS, we search the weight in generalized Jensen-Shannon Divergence in $\{0.1, 0.3, \allowbreak 0.5, 0.7, 0.9\}$.
For DualT, we search the epoch to start reweighting in $\{1, 3, 5, 7, 9\}$.
For SupCon, we search the contrastive loss weighting parameter $\lambda$ in  $\{0.1, 0.3, \allowbreak 0.5, 0.7, 0.9, 1.0\}$ and the temperature parameter $\tau$ in $\{0.1, 0.3, 0.5, 0.7\}$.
For SAM, we search the neighborhood size parameter $\rho$ in $\{0.01, 0.02, 0.05, 0.1, 0.2, \allowbreak 0.5\}$.

\subsection{Comparison with related Hessian-based generalization measures}

A similar Hessian-based measure is proposed in a concurrent work \cite{yang2022does}. The method from \citet{yang2022does} optimizes PAC-Bayesian bounds with the posterior covariance matrix constructed from Hessian's spectrum.
We compare our result from the layer-wise Hessian measure from equation \ref{eq_main} with this alternative Hessian measure, defined as $$\cH:= \max_{(x,y) \in \cD} v^{\top} \bH^+[\ell(f_{\hat{W}}(x), y)] v,$$
where $v$ is the concatenation of the $v_i$ vectors and $\bH^+$ is the truncated Hessian matrix over all model parameters of $v$.
To compute the above measure, we treat $\cD$ as the union of the training and testing data sets.
We use ResNet-18 as the base model, fine tuned on CIFAR-10 and CIFAR-100 data sets. We also evaluate our results of the BERT-Based-Uncased model on MRPC and SST-2 data sets from the GLUE benchmark.
In Table \ref{tab_further_bound_comparison}, we find that the Hessian measure $\sqrt{\cH / n}$ is slightly smaller than the layer-wise Hessian measure from equation \ref{eq_main}.
Rigorously formulating this is an interesting question for future work. We remark that this likely requires an analysis of the Lipschitz-continuity of the full Hessian matrix.

\begin{table}[t!]
    \captionof{table}{Numerical comparison between the layer-wise Hessian measure and the trace of the Hessian of the full model's  loss. The experiments are conducted on ResNet-50 (for CIFAR-10 and CIFAR-100) and BERT-Base (for MRPC and SST2).}\label{tab_further_bound_comparison}
    \centering
    {\small\begin{tabular}{@{}ccccc@{}}
    \toprule
    {Generalization measure} & CIFAR-10 & CIFAR-100 & MRPC & SST2 \\
    \midrule
    $\sum_{i=1}^{L} \frac{\sqrt{\cH_i}}{\sqrt{n}}$ & 0.0853 & 0.1074 & 0.1417 & 0.0271 \\
    $\frac{\sqrt{\cH}}{\sqrt{n}}$ & 0.0396 & 0.0429 & 0.1066 & 0.0047 \\
    $\frac{\sqrt{\sum_{i=1}^{L} \cH_i}}{\sqrt{n}}$ & 0.0144 & 0.0180 & 0.0199 & 0.0034\\    \bottomrule
    \end{tabular}}
\end{table}


\end{document}